%% file: main.tex
\documentclass{article}

\usepackage{amsthm,graphicx,amsfonts,amssymb,bm,url,breakurl,epsfig,epsf,color,fullpage,MnSymbol,mathbbol,fmtcount,algorithm,semtrans,cite,caption,subcaption,multirow}

\usepackage[title]{appendix}

\usepackage[utf8]{inputenc} 
\usepackage[T1]{fontenc}    
\usepackage{url}            
\usepackage{booktabs}       
\usepackage{amsfonts}       
\usepackage{nicefrac}       
\usepackage{microtype}      

\usepackage{algorithm}
\newtheorem{theorem}{Theorem}
\newtheorem{remark}{Remark}
\newtheorem{lemma}{Lemma}

\newtheorem{prop}{Proposition}
\newtheorem{assumption}{Assumption}

\newcommand{\gr}{\nabla}
\newcommand{\grw}{\nabla_{\w}}

\newcommand{\grpsi}{\nabla_{\bpsi}}
\newcommand{\grPsi}{\nabla_{\Psi}}
\newcommand{\sgr}{\tilde{\nabla}}
\newcommand{\sgrw}{\sgr_{\w}} 
\newcommand{\sgrd}{\sgr_{\delta}}
\newcommand{\sgrL}{\sgr_{\Lambda}}
\newcommand{\sgrpsi}{\sgr_{\psi}}
\newcommand{\sgrbpsi}{\sgr_{\bpsi}}

\newcommand{\w}{\bm{w}}
\newcommand{\bpsi}{\bm{\psi}}

\newcommand{\E}{\mathbb{E}}
\newcommand{\wbar}{\overline{\w}}
\newcommand{\sigmaw}{\sigma^2_{\w}}
\newcommand{\sigmapsi}{\sigma^2_{\psi}}

\usepackage[T1]{fontenc}
\usepackage[utf8]{inputenc}
\usepackage{pgfplots}
\usepackage{graphicx}
\usepackage{subcaption}
\usepackage{pstricks}
\usepackage{color}
\usepackage{amsmath}
\usepackage{bbm}
\usepackage{tcolorbox}
\usepackage{tikz}
\usepackage{pgfplots}
\usepackage{wrapfig}
\usepackage{lipsum}  
\usetikzlibrary{positioning}
\usepackage{tcolorbox}
\usepackage{amsfonts,amssymb}
\usepackage{natbib}
\usepackage{mathtools}
\usepackage{commath}
\usepackage{relsize}
\usepackage{bbm}
\usepackage{bm}
\usepackage[font={small}]{caption}
\usepackage{comment,color,soul}
\usepackage{enumerate} 
\usetikzlibrary{arrows,automata}
\usepackage{amsfonts}
\usepackage{url}
\usepackage{lipsum}
\usepackage[thinlines]{easytable}
\usepackage{nicefrac}
\usepackage{algorithm}
\usepackage{algpseudocode}
\usepackage{mysymbol}
\usepackage{float}
\usepackage{titling}

 \makeatletter
\def\BState{\State\hskip-\ALG@thistlm}
\makeatother

\makeatletter
\newcommand{\algmargin}{\the\ALG@thistlm}
\makeatother
\newlength{\whilewidth}
\algdef{SE}[parWHILE]{parWhile}{EndparWhile}[1]
  {\parbox[t]{\dimexpr\linewidth-\algmargin}{%
     \hangindent\whilewidth\strut\algorithmicwhile\ #1\ \algorithmicdo\strut}}{\algorithmicend\ \algorithmicwhile}%
\algnewcommand{\parState}[1]{\State%
\parbox[t]{\dimexpr\linewidth-\algmargin}{\strut #1\strut}}

\makeatletter
\def\Let@{\def\\{\notag\math@cr}}
\makeatother

\DeclarePairedDelimiter\floor{\lfloor}{\rfloor}

\usepackage[hidelinks,backref=page]{hyperref}
\definecolor{darkred}{RGB}{150,0,0}
\definecolor{darkgreen}{RGB}{0,150,0}
\definecolor{darkblue}{RGB}{0,0,150}
\hypersetup{colorlinks=true, linkcolor=darkred, citecolor=darkblue, urlcolor=darkblue}

\renewcommand*{\backref}[1]{}
\renewcommand*{\backrefalt}[4]{%
    \ifcase #1 (Not cited.)%
    \or        (Cited on page~#2.)%
    \else      (Cited on pages~#2.)%
    \fi}
    
\newcommand\NoDo{\renewcommand\algorithmicdo{}}

\title{ \textbf{Robust Federated Learning: \\ The Case of Affine Distribution Shifts}}

\date{}

\makeatletter
\renewcommand*{\@fnsymbol}[1]{\ensuremath{\ifcase#1\or  * \or 1 \or 2 \or 3 \or 4 \else\@ctrerr\fi}}
\makeatother
    
\newcommand*\samethanks[1][\value{footnote}]{\footnotemark[#1]}

\author{
Amirhossein Reisizadeh\thanks{Equal contribution}  \thanks{Department of Electrical and Computer Engineering, UC Santa Barbara, Santa Barbara, CA, USA. \{reisizadeh@ucsb.edu, ramtin@ece.ucsb.edu\}.}
\and 
Farzan Farnia\samethanks[1]  \thanks{Laboratory for Information $\&$ Decision Systems, Massachusetts Institute of Technology, Cambridge,
MA, USA. \{farnia@mit.edu, jadbabai@mit.edu\}.}
\and
Ramtin Pedarsani\samethanks[2]
\and
Ali Jadbabaie\samethanks[3]
}

\begin{document}

\maketitle



\begin{abstract}
\input{0-abstract.tex}
\end{abstract}

\section{Introduction}
\input{1-intro.tex}

\section{Federated Learning Scenario}
\input{3-FL-scenario.tex}

\section{The Proposed \texttt{FedRobust} Algorithm}
\input{4-Federated-Adversarial.tex}

\section{Theoretical Guarantees: Optimization, Generalization and \\ Robustness }
\input{5-gen-opt-SGDA.tex}

\section{Numerical Results}
\input{6-numerical.tex}

\newpage

\newpage
\bibliography{ref}
\bibliographystyle{apalike}


\newpage
\begin{appendices}

\newpage

\input{7-opt-proofs-SGDA.tex}
\end{appendices}

\end{document}

%% file: 0-abstract.tex
Federated learning is a distributed  paradigm for training  models using samples distributed across multiple users in a network, while keeping the samples on users’ devices with the aim of efficiency and  protecting users privacy. In such settings, the training data is often statistically heterogeneous and manifests various distribution shifts across users, which degrades the performance of the learnt model. The primary goal of this paper is to develop a robust federated learning algorithm that achieves satisfactory performance against distribution shifts in users' samples. To achieve this goal, we first consider a \emph{structured} affine distribution shift in users' data that captures the device-dependent data heterogeneity in federated settings. This perturbation model is applicable to various federated learning problems such as image classification where the images undergo device-dependent imperfections, e.g. different intensity, contrast, and brightness. To address affine distribution shifts across users, we propose a \textbf{F}ederated \textbf{L}earning framework \textbf{R}obust to \textbf{A}ffine distribution shifts \texttt{(FLRA)} that is robust against affine distribution shifts to the distribution of observed samples. To solve the \texttt{FLRA}'s distributed minimax optimization problem, we propose a fast and efficient optimization method and provide convergence and performance  guarantees via a gradient Descent Ascent (GDA) method. We further prove generalization error bounds for the learnt classifier to show proper generalization from empirical distribution of samples to the true underlying distribution. We perform several numerical experiments to empirically support \texttt{FLRA}. We show that an affine distribution shift indeed suffices to significantly decrease the performance of the learnt classifier in a new test user, and our proposed algorithm achieves a significant gain in comparison to standard federated learning and adversarial training methods.

%% file: 1-intro.tex
\emph{Federated learning} is a new framework for training a centralized model using data samples distributed over a network of devices, while keeping data localized. Federated learning comes with the promise of training accurate models using local data points such that the privacy of participating devices is preserved; however, it faces several challenges ranging from developing statistically and computationally efficient algorithms to guaranteeing privacy. 

A typical federated learning setting consists of a network of hundreds to millions of devices (nodes) which interact with each other through a central node (a parameter server). Communicating messages over such a large-scale network can lead to major slow-downs due to communication bandwidth bottlenecks \citep{li2019federated,kairouz2019advances}. In fact, the communication bottleneck is one of the main grounds that distinguishes federated and standard distributed learning paradigms. To reduce communication load in federated learning, one needs to depart from the classical setting of distributed learning in which updated local models are communicated to the central server \emph{at each iteration}, and communicate less frequently.

Another major challenge in federated learning is the statistical heterogeneity of training data \citep{li2019federated,kairouz2019advances}. As mentioned above, a federated setting involves many devices,  each generating or storing personal data such as images, text messages or emails. Each user's data samples can have a (slightly) different underlying distribution which is another key distinction between  federated learning and classical learning problems. Indeed, it has been shown that standard federated methods such as \texttt{FedAvg} \citep{mcmahan2016communication} which are designed for i.i.d. data significantly suffer in statistical accuracy or even diverge if deployed over non-i.i.d. samples \citep{karimireddy2019scaffold}. Device-dependency of local data along with privacy concerns in federated tasks does not allow learning the distribution of individual users and necessitates novel algorithmic approaches to learn a classifier robust to distribution shifts across users.  Specifically, statistical heterogeneity of training samples in federated learning can be problematic for generalizing to the distribution of a test node unseen in training time. We show through various numerical experiments that even a simple linear filter applied to the test samples will suffice to significantly degrade the performance of a model learned by \texttt{FedAvg} in standard image recognition tasks.

To address the aforementioned challenges, we propose a new federated learning scheme called \texttt{FLRA}, a \textbf{F}ederated \textbf{L}earning framework with \textbf{R}obustness to \textbf{A}ffine distribution shifts. FLRA has a  small communication overhead and a low computation complexity. The key insight in FLRA is  model the heterogeneity of training data in a device-dependent manner, according to which the samples stored on the $i$th device   $\bbx^i$ are shifted from a ground distribution by an affine transformation $\bbx^i \to \Lambda^i \bbx^i + \delta^i$. 
To further illustrate this point, consider a federated image classification task where each mobile device maintains a collection of images. The images taken by a camera are similarly distorted depending on the intensity, contrast, blurring, brightness and other characteristics of the camera \citep{pei2017deepxplore,hendrycks2019benchmarking}, while these features vary across cameras. In addition to camera imperfections, such unseen distributional shifts also originate from changes in the physical environment, e.g. weather conditions \cite{robey2020model}. Compared to the existing literature, our model provides more robustness compared to the well-known adversarial training models $\bbx^i \to \bbx^i + \delta^i$ with solely additive perturbations \citep{madry2017towards, goodfellow2014explaining, shafahi2018universal}, i.e. $\Lambda^i = I$ . Our perturbation model also generalizes the universal adversarial training approach  in which \emph{all} the training samples are distorted with an identical perturbation $\bbx^i \to \bbx^i + \delta$ \citep{moosavi2017universal}.

Based on the above model, \texttt{FLRA} formulates the robust learning task as a minimax robust optimization problem, which finds a \emph{global} model $\w^*$ that minimizes the total loss induced by the worst-case \emph{local} affine transformations $(\Lambda^{i*}, \delta^{i*})$. One approach to solve this minimax problem is to employ  techniques from adversarial training in which for each iteration and a given global model $\w$, each node optimizes its own local adversarial parameters $(\Lambda^{i}, \delta^{i})$ and a new model is obtained. This approach is however undesirable in federated settings  since it requires extensive computation resources at each device as they need to fully solve the adversarial optimization problem at each iteration. To tackle this challenge, one may propose to use standard distributed learning frameworks in which each node updates its local adversarial parameters and shares with the server at \emph{each iteration} of the distributed algorithm to obtain the updated global model. This is also in contrast with the availability of limited communication resources in federated settings. The key contribution of our work is to develop a novel method called \texttt{FedRobust}, which is a gradient descent ascent (GDA) algorithm to solve the minimax robust optimization problem, can be efficiently implemented in a federated setting, and comes with strong theoretical guarantees. While the \texttt{FLRA} minimax problem is in general non-convex non-concave, we show that \texttt{FedRobust} which alternates between the perturbation and parameter model variables will converge to a stationary point in the minimax objective that satisfies the Polyak-Łojasiewicz (PL) condition. Our optimization guarantees can also be extended to more general classes of non-convex non-concave distributed minimax optimization problems.

As another major contribution of the paper, we use the PAC-Bayes framework \citep{mcallester1999pac,neyshabur2017pac} to prove a generalization error bound for \texttt{FLRA}'s learnt classifier. Our generalization bound applies to multi-layer neural network classifiers and is based on the classifier's Lipschitzness and smoothness coefficients. The generalization bound together with our optimization guarantees suggest controlling the neural network classifier's complexity through Lipschitz regularization methods. Regarding \texttt{FLRA}'s robustness properties, we connect the minimax problem in \texttt{FLRA} to a distributionally robust optimization problem \citep{wiesemann2014distributionally,shafieezadeh2019regularization} where we use an optimal transport cost to measure the distance between distributions. This connection reveals that the \texttt{FLRA}'s minimax objective provides a lower-bound for the objective of a distributionally robust problem. Finally, we discuss the results of several numerical experiments to empirically support the proposed robust federated learning method. Our experiments suggest a significant gain under affine distribution shifts compared to existing adversarial training algorithms. In addition, we show that the trained classifier performs robustly against standard FGSM and PGD adversarial attacks, and outperforms \texttt{FedAvg}. A summary of the key contributions of our work is as follows:
\begin{itemize}
    \item We develop an efficient  federated learning framework that is robust  against affine distribution shifts  using a minimax optimization  approach.
    \item We propose an optimization method to solve the minimax problem  and provide guarantees on the convergence of the iterates in the proposed method to a stationary point.
    \item We Characterize the generalization and robustness properties of our framework.
    \item We Demonstrate the efficiency and advantages of this method compared to the existing    standard approaches via several numerical results.
\end{itemize}


\subsection{Related work}
We divide the literature review to two main lines of work: (i) federated learning and (ii) nonconvex minimix problems and discuss works that are most related to this paper.

As a practical on-device learning paradigm, federated learning has recently gained significant attention in machine learning and optimization communities.  Since the introduction of \texttt{FedAvg} \citep{mcmahan2016communication} as a communication-efficient federated learning method, many works have developed federated methods under different settings with optimization guarantees for a variety of loss functions \citep{haddadpour2019convergence,khaled2020tighter}. Moreover, another line of work has tackled the communication bottleneck in federated learning via compression and sparsification methods \citep{konevcny2016federated,caldas2018expanding,reisizadeh2019fedpaq}. \citep{bhowmick2018protection,geyer2017differentially,li2019differentially,thakkar2019differentially} have focused on designing privacy-preserving federated learning schemes. There have also been several recent works the study local-SGD methods as a subroutine of federated algorithms and provide various convergence results depending on the loss function class \citep{stich2018local,koloskova2019decentralized,wang2018adaptive}. Making federated learning methods robust to non-i.i.d. data has also been the focus of several works \citep{mohri2019agnostic, karimireddy2019scaffold,li2019convergence}. 

Adversarially robust learning paradigms usually involve solving a minimax problem of the form $\min_{\w} \max_{\bpsi} \allowbreak f(\w, \bpsi)$. As the theory of adversarially robust learning surges, there has been thriving recent interests in solving the minimax problem for nonconvex cases. Most recently, \cite{lin2019gradient} provides nonasymptotic analysis for nonconvex-concave settings and shows that the iterates of a simple Gradient Descent Ascent (GDA) efficiently find the stationary points   of the function $\Phi (\w) \coloneqq \max_{\bpsi} f(\w, \bpsi)$. \cite{yang2020global} establishes convergence results for the nonconvex-nonconcave setting and under PL condition. This problem has been studied in the context of game theory as well \citep{nouiehed2019solving}.

%% file: 3-FL-scenario.tex
Consider a federated learning setting with a network of $n$ nodes (devices) connected to a server node. We assume that for every $1\le i\le n$ the $i$th node has access to $m$ training samples in  $S^i=\{(\mathbf{x}^i_j,y^i_j) \in {\mathbb{R}}^{d} \times \mathbb{R}: \, 1\le j\le m\}$. For a given loss function $\ell$ and function class $\ccalF = \{f_{\w} : \w \in \ccalW \}$, the classical federated learning problem is to fit the best model $\w$ to the $nm$ samples via solving the following empirical risk minimization (ERM) problem:
\begin{align}
    \min_{\w \in \ccalW}\;\: \frac{1}{nm} \sum_{i=1}^n \sum_{j=1}^m \: \ell \left( f_{\w}(\mathbf{x}^i_j),y^i_j \right). \nonumber
\end{align} 
As we discussed previously, the training data is statistically heterogeneous across the devices. To capture the non-identically-distributed nature of data in federated learning, we assume that the data points of each node have a local distribution shift from a common distribution. To be more precise, we assume that each sample stored in node $i$ in $S^i$ is distributed according to an affine transformation $h^i$ of a universal underlying distribution $P_{\mathbf{X},Y}$, i.e., transforming the features of a sample $(\mathbf{x},y)\sim P_{\mathbf{X},Y}$ according to the following affine function 
\begin{align}
    h^i(\bbx)\, :=\, \Lambda^i \bbx + \delta^i. \nonumber
\end{align} 
Here $\Lambda^i\in{\mathbb{R}}^{d\times d}$ and  $\delta^i\in {\mathbb{R}}^{d}$, with $d$ being the dimension of input variable $\bbx$, characterize the affine transformation $h^i$ at node $i$. According to this model, all samples stored at node $i$ are affected with the same affine transformation while other nodes $j \neq i$ may experience different transformations.

This \emph{structured} model particularly supports the data heterogeneity in federated settings. That is, the data generated and stored in each federated device is exposed to identical yet device-dependent distortions while different devices undergo different distortions. As an applicable example that manifests the proposed perturbation model, consider a federated image classification task over the images taken and maintained by mobile phone devices. Depending on the environment's physical conditions and the camera's imperfections, the pictures taken by a particular camera undergo device-dependent perturbations. According to the proposed model, such distribution shift is captured as an affine transformation $h^i(\bbx) \! = \! \Lambda^i \bbx + \delta^i$ on the samples maintained by node $i$. To control the perturbation power, we consider bounded Frobenius and Euclidean  norms $\Vert\Lambda - I_{d}\Vert_F \le \epsilon_1$ and $\Vert\delta\Vert_2\le \epsilon_2$  enforcing the affine transformation to have a bounded distance from the identity transformation.




Based on the model described above, our goal is to solve the following distributionally robust federated learning problem:
\begin{align} \label{eq:ERM}
    \min_{\w \in \ccalW}\;\: 
    \frac{1}{n}\sum_{i=1}^n\: \,
    \max_{\scriptstyle \Vert \Lambda^i - I \Vert_F \leq \epsilon_1 \atop \scriptstyle \Vert \delta^i \Vert \leq \epsilon_2}\;\: \frac{1}{m}\sum_{j=1}^m\: \ell \left( f_{\w}(\Lambda^i \bbx^i_j +\delta^i), y^i_j \right). 
\end{align}
The minimax problem \eqref{eq:ERM} can be interpreted as $n+1$ coupled optimization problems. First, in $n$ inner local maximization problems and for a given global model $\w$, each node $1 \leq i \leq n$ seeks a (feasible) affine transformation $(\Lambda^i, \delta^i)$ which results in high losses via solving
\begin{align}
    \max_{\scriptstyle \Vert \Lambda^i - I \Vert_F \leq \epsilon_1 \atop \scriptstyle \Vert \delta^i \Vert \leq \epsilon_2}\;\: \frac{1}{m}\sum_{j=1}^m\: 
    \ell \left( f_{\w}(\Lambda^i \bbx^i_j +\delta^i), y^i_j \right) \nonumber
\end{align}
over its $m$ training samples in $S^i$. Then, the outer minimization problem finds a global model yielding the smallest value of cumulative losses over the $n$ nodes.

Solving the above minimax problem requires collaboration of distributed nodes via the central server. In federated learning paradigms however, such nodes are entitled to limited computation and communication resources. Such challenges particularly prevent us from employing the standard techniques in adversarial training and distributed ERM. More precisely, each iteration of adversarial training  requires solving a maximization problem at each local node which incurs extensive computational cost. On the other hand, tackling the minimax problem \eqref{eq:ERM} via iterations of standard distributed learning demands frequent message-passing between the nodes and central server \emph{at each iteration}, hence yielding massive communication load on the network. To account for such system challenges, we constitute our goal to solve the robust minimax problem in \eqref{eq:ERM} with small computation and communication cost so that it can be feasibly and efficiently implemented in a federated setting.



%% file: 4-Federated-Adversarial.tex
To guard against affine distribution shifts, we propose to change the original constrained maximization problem to the following worst-case loss at each node $i$, given a Lagrange multiplier  $\lambda \! > \! 0$: 
\begin{equation} \label{Eq: AFL def}
   \max_{\Lambda^i,\delta^i} \: f^i(\w,\Lambda^i,\delta^i) \;
    \coloneqq \; \max_{\Lambda^i,\delta^i} \: \frac{1}{m} \sum_{j=1}^m \: \ell \left( f_{\w}(\Lambda^i \bbx^i_j +\delta^i), y^i_j \right)
    -
    \lambda \Vert \Lambda^i - I \Vert^2_F -\lambda \Vert \delta^i \Vert^2_2.  
\end{equation}
Here we use a norm-squared penalty requiring a bounded distance between the feasible affine transformations and the identity mapping, and find the worst-case affine transformation that results in the maximum loss for the samples of node $i$. By averaging such worst-case local losses over all the $n$ nodes and minimizing w.r.t. model $\w$, we reach the following minimax optimization problem:
\begin{equation} \label{eq: FLRA}
    \min_{\w \in \ccalW}\;\:
    \max_{(\Lambda^i,\delta^i)_{i=1}^n}\;\:
    \frac{1}{nm}\sum_{i=1}^n \sum_{j=1}^m\; \ell \left( f_{\w}(\Lambda^i \bbx^i_j +\delta^i), y^i_j \right) 
    -
    \lambda\Vert \Lambda^i - I \Vert^2_F -\lambda \Vert \delta^i \Vert^2_2.  
\end{equation}
This formalizes our approach to tackling the robust federated learning problem, which we call ``\textbf{F}ederated \textbf{L}earning framework \textbf{R}obust to \textbf{A}ffine distribution shift'' or \texttt{FLRA} in short. 
\begin{algorithm}[b!] 
\caption{\texttt{FedRobust}}\label{Alg: FedRobust}
\hspace*{\algorithmicindent} \textbf{Input:} Initialization $\{\w^i_0=\w_0, \Lambda^i_0, \delta^i_0\}_{i=1}^n$,  step-sizes $\eta_1, \eta_2$, number of local updates $\tau$, total number of iterations $T$
\begin{algorithmic}[1]
\NoDo
\For{each iteration $t = 0, \cdots, T-1$ and each node $i \in [n] = \{1, \cdots, n\}$}
\State node $i$ computes stochastic gradients $\sgrL f^i$ and $\sgrd f^i$ and updates
    \begin{align} \label{eq: ascent update rule}
        \Lambda^i_{t+1}
        &=
        \Lambda^i_t + \eta_2 \sgrL f^i(\w^i_t, \Lambda^i_t, \delta^i_t) 
        \\
        \delta^i_{t+1} 
        &=
        \delta^i_t + \eta_2 \sgrd f^i(\w^i_t, \Lambda^i_t, \delta^i_t) \nonumber
    \end{align}
    \If{$t$ does not divide $\tau$}
    \State node $i$ computes $\sgrw f^i$ and updates
    \begin{align}
        \w^i_{t+1} &= \w^i_t - \eta_1 \sgrw f^i(\w^i_t, \Lambda^i_t, \delta^i_t) \nonumber
    \end{align} 
    \Else
    \State node $i$ computes $\w^i_t  -  \eta_1 \sgrw f^i ( \w^i_t, \Lambda^i_t, \delta^i_t )$ and uploads to server
    \State server aggregates, takes the average and sends to all nodes $i$:
    \begin{equation}
        \w^i_{t+1} = \frac{1}{n} \sum_{j=1}^{n} \left[ \w^j_t - \eta_1 \sgrw f^j(\w^j_t, \Lambda^j_t, \delta^j_t) \right] \nonumber
    \end{equation}
    \EndIf
\EndFor
\end{algorithmic}
\hspace*{\algorithmicindent} \textbf{Output:} $\wbar_T = \frac{1}{n} \sum_{i=1}^{n} \w^i_T$
\end{algorithm}

In order to solve \texttt{FLRA} in \eqref{eq: FLRA}, we propose a gradient optimization method that is  computationally  and communication-wise efficient, called \texttt{FedRobust}. The proposed \texttt{FedRobust} algorithm is an iterative scheme that applies stochastic gradient descent ascent (SGDA) updates for solving the minimax problem \eqref{eq: FLRA}. As summarized in Algorithm \ref{Alg: FedRobust}, in each iteration $t$ of local updates, each node $i$ takes a (stochastic) gradient ascent step and updates its affine transformation parameters $(\Lambda^i_t, \delta^i_t)$. It also updates the local classifier's parameters $\w^i_t$ via a gradient descent step. After $\tau$ local iterations, local models $\w^i_t$ are uploaded to the server node where the global model is obtained by averaging the local ones. The averaged model is then sent back to the nodes to begin the next round of local iterations with this fresh initialization. Note that each node updates its perturbation parameters only once in each iteration which yields light computation cost as opposed to standard adversarial training methods. Moreover, periodic communication at every $\tau$ iterations, reduces the communication load compared to standard distributed optimization methods by a factor $\tau$.


It is worth noting that the local affine transformation variables $\Lambda^i, \delta^i$ are coupled even though they remain on their corresponding nodes and are not exchanged with the server. This is due to the fact that the fresh model $\w$ is the average of the updated models from \emph{all} the nodes; hence, updating $\Lambda^i, \delta^i$ for node $i$ will affect $\Lambda^j, \delta^j$ for other nodes $j \neq i$ in the following iterations. This is indeed a technical challenge that arises in proving the optimization guarantees of \texttt{FedRobust} in Section \ref{subsec: optimization}.

%% file: 5-gen-opt-SGDA.tex
In this section, we establish the main guarantees of the proposed \texttt{FLRA} formulation and the optimization algorithm \texttt{FedRobust}. First, we characterize the convergence of \texttt{FedRobust} in Algorithm \ref{Alg: FedRobust} to solve the minimax problem \eqref{eq: FLRA}. Next, we prove that the learned hypothesis will properly generalize from training data to unseen test samples. Lastly, we demonstrate that solving the \texttt{FLRA}'s minimax problem \eqref{eq: FLRA} results in a robust classifier to Wasserstein shifts structured across the nodes.

\subsection{Optimization guarantees} \label{subsec: optimization}

In this section, we establish our main convergence results and show that \texttt{FedRobust} finds saddle points of the minimax problem in \eqref{Eq: AFL def} for two classes of loss functions. We first set a few notations as follows. We let matrix $\bpsi^i = (\Lambda^i, \delta^i) \in \reals^{d \times (d+1)}$ denote the joint transformation variables corresponding to node $i$. The collection of $n$ such variables corresponding to the $n$ nodes is denoted by the matrix $\Psi = (\bpsi^1;\cdots;\bpsi^n)$. We can now rewrite the minimax problem \eqref{eq: FLRA} as follows:
\begin{gather} \label{eq: min max f}
    \min_{\w} \max_{\Psi} f(\w, \Psi) 
    \coloneqq
    \min_{\w} \max_{\bpsi^1, \cdots, \bpsi^n} \frac{1}{n} \sum_{i=1}^{n} f^i(\w,\bpsi^i),
\end{gather}
where $f$ and $f^i$s denote the penalized global and local losses, respectively; that is, for each node $i$ 
\begin{align} \label{eq: local fi}
    f^i(\w,\bpsi^i)
    \coloneqq
    \frac{1}{m} \sum_{j=1}^{m} \ell \left( f_{\w} (\Lambda^i\bbx^i_j + \delta^i),y^i_j \right) - \lambda \Vert \Lambda^i - I \Vert_F^2 - \lambda \Vert \delta^i \Vert^2.
\end{align}
We also define $\Phi(\w) \coloneqq \max_{\Psi} f(\w, \Psi)$ and $\Phi^* \coloneqq \min_{\w} \Phi(\w)$. Next, we state a few customary assumptions on the data and loss functions. As we mentioned before, we assume that data is heterogeneous (non-iid). There are several notions to quantify the degree of heterogeneity in the data. In this work we use a notion called \emph{non-iid degree} which is defined as the variance of the local gradients with respect to a global gradient \citep{yu2019linear}. 

The next two assumptions impose customary conditions on the gradients of local functions.

\begin{assumption}[Bounded non-iid degree]\label{assumption:bounded-degree}
We assume that when there are no perturbations, the variance of the local gradients with respect to the global gradient is bounded. That is, there exists $\rho_f^2$ such that
\begin{gather} 
    \frac{1}{n} \sum_{i=1}^{n} \norm{\grw f^i(\w, \bpsi^i) - \grw f(\w, \Psi)}^2 
    \leq
    \rho_f^2,
    \quad
    \text{ for } \bpsi^i=(I,0), \Psi = (\bpsi^1;\cdots;\bpsi^n), \text{ and } \forall \w. \nonumber
\end{gather}
\end{assumption}
\begin{assumption}[Stochastic gradients]\label{assumption:stoch-gradients}
For each node $i$, the stochastic gradients $\sgrw f^i$ and $\sgrbpsi f^i$ are unbiased and have variances bounded by $\sigma^2_{\w}$ and $\sigma^2_{\bpsi}$, respectively. That is,
\begin{align} 
    \E \norm{\sgrw f^i(\w, \bpsi) - \grw f^i(\w, \bpsi)}^2 
    \leq
    \sigma^2_{\w}, \quad
    \E \norm{\sgrbpsi f^i(\w, \bpsi) - \grpsi f^i(\w, \bpsi)}^2 
    \leq
    \sigma^2_{\bpsi},
    \quad
    \forall \w, \bpsi. \nonumber
\end{align}
\end{assumption}

\begin{assumption}[Lipschitz gradients] \label{assumption:smooth}
All local loss functions have Lipschitz gradients. That is, for any node $i$, there exist constants $L_1, L_2, L_{12}$, and $L_{21}$ such that for any $\w, \w', \bpsi, \bpsi'$ we have
\begin{gather}
    \norm{\grw f^i(\w, \bpsi) - \grw f^i(\w', \bpsi)} 
    \leq
    L_1 \norm{\w - \w'}, \quad
    \norm{\grw f^i(\w, \bpsi) - \grw f^i(\w, \bpsi')} 
    \leq
    L_{12} \norm{\bpsi - \bpsi'}_F, \\
    \norm{\grpsi f^i(\w, \bpsi) - \grpsi f^i(\w', \bpsi)}_F 
    \leq 
    L_{21} \norm{\w - \w'}, \quad
    \norm{\grpsi f^i(\w, \bpsi) - \grpsi f^i(\w, \bpsi')}_F 
    \leq
    L_2 \norm{\bpsi - \bpsi'}_F . \nonumber
\end{gather}
\end{assumption}
We show the convergence of \texttt{FedRobust} for two classes of loss functions: PL-PL and nonconvex-PL. Next, we briefly describe these classes and state the main results. The celebrated work of Polyak \citep{polyak1963gradient} introduces a sufficient condition for an unconstrained minimization problem $\min_x g(x)$ under which linear convergence rates can be established using gradient methods. A function $g(x)$ satisfies the Polyak-Łojasiewicz (PL) condition if $g^* = \min_x g(x)$ exits and is bounded, and there exists a constant $\mu > 0$ such that $\Vert \gr g(x) \Vert^2 \geq 2 \mu (g(x) - g^*), \, {\forall} x$. Similarly, we can define two-sided PL condition for our minimax objective function in \eqref{eq: min max f} \citep{yang2020global}. 
\begin{assumption}[PL condition]\label{assumption:PL}
The global function $f$ satisfies the two-sided PL condition, that is, there exist positive constants $\mu_1$ and $\mu_2$ such that
\begin{gather}
    {\normalfont{\text{(i)}}} \quad \frac{1}{2 \mu_1} \norm{\grw f(\w, \Psi)}^2
    \geq
    f(\w, \Psi) - \min_{\w} f(\w, \Psi), \\
    {\normalfont{\text{(ii)}}} \quad \frac{1}{2 \mu_2} \norm{\grPsi f(\w, \Psi)}^2_F
    \geq
    \max_{\Psi} f(\w, \Psi) - f(\w, \Psi). \nonumber
\end{gather}
\end{assumption}
In other words, Assumptions \ref{assumption:PL} states that the functions $f(\cdot, \Psi)$ and $-f(\w, \cdot)$ satisfy the PL condition with constants, $\mu_1$ and $\mu_2$, respectively. To measure the optimality gap at iteration $t$, we define the potential function $P_t \coloneqq a_t + \beta b_t$, where
\begin{align}
    a_t \coloneqq \E [ \Phi(\wbar_t) ] - \Phi^* 
    \quad \text{ and } \quad 
    b_t \coloneqq \E [\Phi(\wbar_t) - f(\wbar_t, \Psi_t)], \nonumber
\end{align}
and $\beta$ is an arbitrary and positive constant. Note that both $a_t$ and $b_t$ are non-negative and if $P_t$ approaches zero, it implies that $(\wbar_t, \Psi_t)$ is approaching a minimax point.
\begin{theorem}[PL-PL loss] \label{Thm: PL-PL convergence}
Consider the iterates of {\normalfont \texttt{FedRobust}} in Algorithm \ref{Alg: FedRobust} 
and let Assumptions  \ref{assumption:bounded-degree}, \ref{assumption:smooth}, and \ref{assumption:PL} hold.
Then for  any iteration $t \geq 0$, the optimality gap $P_t \coloneqq a_t + \frac{1}{2} b_t$ satisfies the following:
\begin{align} 
    P_{t} 
    &\leq
    \left( 1 - \frac{1}{2} \mu_1 \eta_1 \right)^t P_0
    +
    32 \eta_1 \frac{\tilde{L}}{\mu_1} (\tau - 1)^2 \rho^2
    +
    8 \eta_1 \frac{\tilde{L}}{\mu_1}  (\tau - 1) (n + 1) \frac{\sigmaw}{n}
    +
    \eta_1 \frac{\hat{L}}{\mu_1} \frac{\sigmaw}{n}
    +
    \frac{\eta_2^2}{\eta_1} \frac{L_2}{2 \mu_1} \sigmapsi, \nonumber
\end{align}
for maximization step-size $\eta_2$ and minimization step-size $\eta_1$ that satisfy the following conditions:
\begin{gather} \label{eq: Thm PL-PL conditions}
    \eta_2 
    \leq
    \frac{1}{L_2}, \quad
    32 \eta_1^2 (\tau - 1)^2 L_1^2 
    \leq
    1, \quad
    \frac{\mu^2_2 \eta_2 n}{\eta_1 L_1 L_2}
    \geq
    1 + 8 \frac{L_{12}^2}{L_1 L_2},\quad
    \eta_1 \left( \hat{L}  + \frac{80 \tilde{L} (\tau - 1) }{\mu_1 \eta_1 ( 1 - \frac{1}{2} \mu_1 \eta_1 )^{\tau-1}} \right) 
    \leq
    1. \nonumber 
\end{gather}
Here, we denote $\rho^2 \coloneqq 3\rho_f^2 + 6 L_{12}^2 (\epsilon_1^2 + \epsilon_2^2)$ where $\epsilon_1$ and $\epsilon_2$ specify the bounds on the affine transformations $h^i(\bbx) = \Lambda^i \bbx + \delta^i$. We also use the following notations:
\begin{gather} \label{eq: tildeL hatL}
    L_{\Phi} = L_1 + \frac{L_{12} L_{21}}{2n\mu_2}, \quad
    \tilde{L} = \frac{3}{2} \eta_1 L_1^2 + \frac{1}{2}  \eta_2 L_{21}^2, \quad
    \hat{L} =  \frac{3}{2} L_{\Phi} + \frac{1}{2} L_1 + \frac{L_{21}^2}{L_2}.  \nonumber
\end{gather} 
\end{theorem} 

\begin{proof}
We provide the proof of Theorem \ref{Thm: PL-PL convergence} for any $\beta \leq 1/2$ in Appendix \ref{proof: Thm PL-PL}.
\end{proof}
 
Let us denote $L \coloneqq \max\{L_1, L_2/n, L_{12}/\sqrt{n}, L_{21}/\sqrt{n}\}$, $\mu \coloneqq \min\{\mu_1, \mu_2\}$ and define the condition number $\kappa \coloneqq L / \mu$. Then for feasible step-sizes $\eta_1$ and $\eta_2$ we have
\begin{align}
    P_{t} 
    \leq
    e^{- \frac{1}{2} \mu \eta_1 t} P_0 
    +
    \ccalO \left( \eta_1^2 + n \eta_1 \eta_2 \right) \kappa L (\tau - 1)^2 \rho^2
    +
    \ccalO \left( \eta_1^2 + n \eta_1 \eta_2 \right) \kappa L (\tau - 1)  \sigmaw 
    +
    \ccalO \left( \eta_1 \right) \kappa^2 \frac{\sigmaw}{n}
    +
    \ccalO \left( \frac{\eta_2^2}{\eta_1} \right) n \kappa \sigmapsi. \nonumber
\end{align}
Special cases of this convergence result is consistent with similar ones already established in the literature. In particular the case of regular (non-federated) distributed optimization i.e. when $\tau = 1$, Theorem \ref{Thm: PL-PL convergence} recovers the convergence result in \cite{yang2020global} for a minimax problem with PL-PL cost functions. As another special case of our result, putting $\epsilon_1, \epsilon_2 \to 0$ reduces the problem to standard (non-robust) federated learning where our result is consistent with the prior work as well. In particular, setting  $\epsilon_1, \epsilon_2 \to 0$ and consequently $\eta_2 \to 0$ in this result recovers standard federated learning convergence rates for PL losses \citep{haddadpour2019convergence}.


Next, we relax the PL condition on $f(\cdot, \Psi)$ stated in Assumption \ref{assumption:PL} (i) and show that the iterates of the \texttt{FedRobust} method find a stationary point of the minimax problem \eqref{eq: min max f} when the objective function $f(\w, \Psi)$ only satisfies the PL condition with respect to $\Psi$ and is nonconvex with respect to $\w$.

\begin{theorem}[Nonconvex-PL loss] \label{Thm: nonconvex-PL convergence}
Consider the iterates of {\normalfont \texttt{FedRobust}} in Algorithm \ref{Alg: FedRobust} and let Assumptions  \ref{assumption:bounded-degree}, \ref{assumption:smooth}, and \ref{assumption:PL} (ii) hold.
Then, the iterates of {\normalfont \texttt{FedRobust}} after $T$ iterations satisfy:
\begin{align} \label{eq: Thm nonconvex-PL}
    \frac{1}{T} \sum_{t=0}^{T-1} \E \norm{\gr \Phi( \wbar_{t} )}^2
    &\leq
    \frac{4 \Delta_{\Phi}}{\eta_1 T}  
    +
    \frac{4 L_2^2}{\mu_2^2 n^2}\frac{ \epsilon^2}{\eta_1 T} 
    +
    64 \eta_1 \tilde{L} (\tau - 1)^2 \rho^2 
    +
    16 \eta_1 \tilde{L}  (\tau - 1)  \frac{n + 1}{n} \sigmaw
    +
    2 \eta_1 \hat{L} \frac{\sigmaw}{n}
    +
    \frac{\eta_2^2}{\eta_1} L_2 \sigmapsi, \nonumber
\end{align}
with $\tilde{L},\hat{L},L_{\Phi}, \rho^2 $ defined in Theorem \ref{Thm: PL-PL convergence}, $\epsilon^2 \coloneqq \epsilon_1^2 + \epsilon_2^2$ and $\Delta_{\Phi} \coloneqq \Phi(\w_0) - \Phi^*$, if step-sizes $\eta_1, \eta_2$ satisfy
\begin{gather}
    \eta_2 
    \leq
    \frac{1}{L_2}, \quad
    \frac{\eta_1}{\eta_2} 
    \leq
    \frac{\mu_2^2 n^2}{8 L_{12}^2}, \quad
    32 \eta_1^2 (\tau - 1)^2 L_1^2 
    \leq
    1, \quad
    \eta_1 \left( \hat{L}  + 40 \tilde{L} (\tau - 1)^2 \right)
    \leq 1. \nonumber 
\end{gather}
\end{theorem}


\begin{proof}
We defer the proof of Theorem \ref{Thm: nonconvex-PL convergence} to Appendix \ref{appendix: proof of Thm nonconvex-PL}. 
\end{proof}

Theorem \ref{Thm: nonconvex-PL convergence} implies that after $T$ iterations of \texttt{FedRobust}, there exists $0 \leq t \leq T-1$ for which we have
\begin{align}
    \E \norm{\gr \Phi(\wbar_{t})}^2 
    &\leq 
    \ccalO \left(\frac{\Delta_{\Phi}  + \kappa^2 \left( \epsilon_1^2 + \epsilon_2^2 \right)}{\eta_1 T} \right) 
    +
    \ccalO \left( \eta_1^2 + n \eta_1 \eta_2 \right) L^2 (\tau - 1)^2 \rho^2 \\
    &\quad +
    \ccalO \left( \eta_1^2 + n \eta_1 \eta_2 \right) L^2 (\tau - 1) \sigmaw
    +
    \ccalO \left( \eta_1 \right) \kappa L \frac{\sigmaw}{n}
    +
    \ccalO \left( \frac{\eta_2^2}{\eta_1} \right) n L \sigmapsi, \nonumber
\end{align}
which yields that the averaged model $\wbar_{t}$ approaches a stationary saddle point of $\Phi(\w)$ for proper choices of the step-sizes. It is worth noting that similar to Theorem \ref{Thm: PL-PL convergence}, this result recovers existing results in the literature for the special cases of distributed minimax optimization, i.e. $\tau=1$ \citep{lin2019gradient} and standard federated learning for nonconvex objectives, i.e. when $\epsilon_1, \epsilon_2 \to 0$ \citep{wang2018adaptive,reisizadeh2019fedpaq}.

\subsection{Generalization guarantees}

Following the margin-based generalization bounds developed in \cite{bartlett2017spectrally,neyshabur2017pac,farnia2018generalizable}, we consider the following margin-based error measure for analyzing the generalization error in \texttt{FLRA} with general neural network classifiers:
\begin{equation}\label{Generalization Error Measure}
    \mathcal{L}^{\operatorname{adv}}_{\gamma}(\w) 
    \coloneqq
    \frac{1}{n}\sum_{i=1}^n {\Pr}_i \left( f_{\w}(h^i_{adv}(\mathbf{X}) )[Y] - \max_{j\neq Y} f_{\w}(h^i_{adv}(\mathbf{X}) )[j] \le \gamma \right).
\end{equation}
Here, $h^i_{adv}$ denotes the worst-case affine transformation for node $i$ in the maximization problem \eqref{Eq: AFL def}; $\Pr_i$ denotes the probability measured by the underlying distribution of node $i$, and $f_{\w}(\mathbf{x})[j]$ denotes the output of the neural network's last softmax layer for label $j$. Note that for $\gamma=0$, the above definition reduces to the average misclassfication rate under the distribution shifts, which we simply denote by $\mathcal{L}^{\operatorname{adv}}(\w)$. We also use $\hat{\mathcal{L}}^{\operatorname{adv}}_\gamma(\w)$ to denote the above margin risk for the empirical distribution of samples, where we replace the underlying $\Pr_i$ with $\hat{\Pr}_i$ being the empirical probability evaluated for the $m$ samples of node $i$. The following theorem bounds the difference of the empirical and underlying margin-based error measures in \eqref{Generalization Error Measure} for a general deep neural network function. The bound is based on the spectral norms of the weight matrices across layers which provide upper-bounds for the Lipschitz and smoothness coefficients of the neural network.

\begin{theorem}\label{Thm: generalization}
Consider an $L$-layer neural network with $d$ neurons per layer. We assume the activation function of the neural network $\sigma$ satisfies $\sigma(0)=0$ and $\max_t\{|\sigma'(t)|,|\sigma''(t)|\}\le 1$. Suppose the same Lipschitzness and smoothness condition holds for loss $\ell$, and $\Vert\mathbf{X}\Vert_2\le B$. We assume the weights of the neural network are spectrally regularized such that for $M>0$:
\begin{align}
    \frac{1}{M} \le \left( \prod_{i=1}^d\Vert\w_i \Vert_{\sigma} \right)^{1/d} 
    \le
    M, \nonumber
\end{align}
with $\Vert\cdot\Vert_{\sigma}$ denoting the maximum singular value, i.e., the spectral norm. Also, suppose that for $\eta>0$, 
\begin{align}
\operatorname{Lip}(\nabla f_{\w}) := \sum_{i=1}^d\prod_{j=1}^i\Vert\w_i \Vert_\sigma \le \lambda(1-\eta)\nonumber
\end{align}
holds where $\operatorname{Lip}(\nabla f_{\w})$ upper-bounds the Lipschitz coefficient of the gradient $\nabla_{\mathbf{x}}\ell(f_{\w}(\mathbf{x},y))$. Then, for every $\xi>0$ with probability at least $1 - \xi$ the following holds for all feasible weights $\w$:
\begin{align}
    \mathcal{L}^{\operatorname{adv}}(\w) - \hat{\mathcal{L}}^{\operatorname{adv}}_{{\gamma}}(\w)
    \le
    \mathcal{O} \left( \sqrt{ \frac{ B^2L^2d\log(Ld)\lambda^2\bigl(\prod_{i=1}^L\Vert \w_i\Vert_{\sigma} \sum_{i=1}^L\frac{\Vert \w_i\Vert^2_F}{\Vert \w_i\Vert^2_{\sigma}}\bigr)^2+L\log\frac{nmL\log(M)}{\eta\xi}  }{m \gamma^2 (\lambda-(1+B)\operatorname{Lip}(\nabla f_{\w}))^2} }\right). \nonumber
\end{align}
\end{theorem}
\begin{proof}
We defer the proof to Appendix \ref{appendix: proof of Thm generalization}.
\end{proof}
This theorem gives a non-asymptotic bound on the generalization risk of \texttt{FLRA} for spectrally regularized neural nets with their smoothness constant bounded by $\lambda$. Thus, we can control the generalization performance by properly regularizing the Lipschitzness and smoothness degrees of the neural net. Note that this result requires a smooth and Lipschitz activation function in the neural network, such as the exponential linear unit (ELU) activation. In our numerical experiments, we also tried the popular ReLU activation, which does not satisfy the smoothness condition. However, we still observed a satisfactory generalization performance in those experiments, indicating that the above guarantee can practically extends to ReLU-type non-linearities as well.

\subsection{Distributional robustness}

To analyze \texttt{FLRA}'s robustness properties, we draw a connection between \texttt{FLRA} and distributionally robust optimization using optimal transport costs. Consider the optimal transport cost $W_c(P,Q)$ for quadratic cost $c(\mathbf{x},\mathbf{x}')=\frac{1}{2}\Vert \mathbf{x}-\mathbf{x}'\Vert^2_2$ defined as
\begin{align}
W_c(P,Q) \coloneqq \min_{M\in \Pi(P,Q)} \: \E [c(\mathbf{X},\mathbf{X}') ], \nonumber
\end{align}
where $\Pi(P,Q)$ denotes the set of all joint distributions on $(\mathbf{X},\mathbf{X}')$ with marginal distributions $P,Q$. In other words, $W_c(P,Q)$ measures the minimum expected cost for transporting samples between $P$ and $Q$. In order to define a distributionally robust federated learning problem against affine distribution shifts, we consider the following minimax problem:
\begin{equation} \label{Eq: DRO Federeted Learning}
    \min_{\w}\;\frac{1}{n}\sum_{i=1}^n \max_{\Lambda^i,\delta^i}\;\bigl\{ \mathbb{E}_{P^i}\bigl[\ell\bigl(f_{\w} (\Lambda^i\mathbf{X}+\delta^i),Y\bigr) \bigr] -  W_c(P^i_{\mathbf{X}},P^i_{\Lambda^i\mathbf{X}+\delta^i})\bigr\}.   
\end{equation}
In this distributionally robust learning problem, we include a penalty term controlling the Wasserstein cost between the original distribution of node $i$ denoted by $P^i$ and its perturbed version under an affine distribution shift, i.e., $P^i_{\Lambda^i\mathbf{X}+\delta^i}$. Note that here we use the averaged Wasserstein cost 
\begin{align}
    \frac{1}{n}\sum_{i=1}^n W_c(P^i_{\mathbf{X}},P^i_{\Lambda^i\mathbf{X}+\delta^i}) \nonumber
\end{align}
to measure the distribution shift caused by the affine shifts $(\Lambda^i,\delta^i)_{i=1}^n$. The following theorem shows that this Wasserstein cost can be upper-bounded by a norm-squared function of $\Lambda$ and $\delta$ that appears in the \texttt{FLRA}'s minimax problem.
\begin{theorem}\label{Thm: DRO}
Consider the Wasserstein cost $W_c(P_{\mathbf{X}},P_{\Lambda\mathbf{X}+\delta})$ between the distributions of $\mathbf{X}$ and its affine perturbation $\Lambda\mathbf{X}+\delta$. Assuming $ \Vert \E [\mathbf{X}\mathbf{X}^T] \Vert_{\sigma} \le \lambda$, we have
\begin{align} \label{eq: Thm DRO}
    W_c(P_{\mathbf{X}},P_{\Lambda\mathbf{X}+\delta}) \le \max\{\lambda,1\} \bigl[\Vert \Lambda -I \Vert^2_F + \Vert \delta\Vert^2_2 \bigr]. 
\end{align}
\end{theorem}
\begin{proof}
We defer the proof to Appendix \ref{appendix: proof of Thm DRO}.
\end{proof}

Substituting the Wasserstein cost in \eqref{Eq: DRO Federeted Learning} with the upper-bound \eqref{eq: Thm DRO} results in the \texttt{FLRA}'s minimax \eqref{eq: FLRA}. As a result, if 
\begin{align}
    \frac{\lambda}{n} \sum_{i=1}^n[ \Vert\Lambda^i - I\Vert^2_F + \Vert\delta^i\Vert^2_2] \le \varepsilon^2 \nonumber
\end{align}
holds for the optimized $\Lambda^i,\delta^i$'s, we will also have the averaged Wasserstein cost bounded by
\begin{align}
    \frac{1}{n} \sum_{i=1}^n W_c  (P^i_{\mathbf{X}},P^i_{\Lambda^i\mathbf{X}+\delta^i}) \le \varepsilon^2. \nonumber
\end{align}
Theorem \ref{Thm: DRO}, therefore, shows the \texttt{FLRA}'s minimax approach optimizes a lower-bound on the distributionally robust \eqref{Eq: DRO Federeted Learning}.

%% file: 6-numerical.tex
We implemented \texttt{FedRobust} in the Tensorflow platform \citep{abadi2016tensorflow} and numerically evaluated the algorithm's robustness performance against affine distribution shifts and adversarial perturbations. We considered the standard MNIST \citep{lecun1998mnist} and CIFAR-10 \citep{krizhevsky2009learning} datasets and used three standard neural network architectures in the literature: AlexNet \citep{krizhevsky2012imagenet}, Inception-Net \citep{szegedy2015going}, and a mini-ResNet \citep{he2016deep}. 

\subsection{CIFAR-10 data: Experimental setup}
In the experiments, we simulated a federated learning scenario with $n=10$ nodes where each node observes $m=5000$ training samples. We also divided the extra $10,000$ samples in each dataset to two validation and test sets containing $5000$ samples each. For CIFAR-10 samples, we applied the sandard normalization and scaled and linearly mapped the pixel intensity values to interval $[-1,1]$. We applied batch normalization \cite{ioffe2015batch} in order to stabilize training and used the ADAM optimizer \citep{kingma2014adam} with stepsize value $10^{-4}$ and default beta parameters $\beta_1=0.9$ and $\beta_2=0.99$ to optimize the neural net's parameters for $T=100$ epochs ($10000$ iterations). 

We did cross validation to choose $\lambda\in \{0.1,0.5,1,5,10,50\} $ and chose the $\lambda$-value resulting in the closest additive penalty $\frac{1}{n}\sum_{i=1}^n[ \Vert\Lambda^{i^*}-I\Vert_2^2 + \Vert\delta^{i^*}\Vert_2^2]$ to $10$ percent of the average sample norm, i.e. $\frac{0.1}{m}\sum_{i=1}^m\Vert\mathbf{x}^{\operatorname{val}}_i \Vert^2_2$, over the $m=5000$ validation samples. To perform GDA optimization, we applied two ascent steps per descent step with stepsize $\frac{1}{2\lambda}$. In order to simulate an affine distribution shift, we manipulated each $\tilde{\mathbf{x}}^i_{j}$ in the original training dataset via an affine transformation chosen randomly at each node:
\begin{equation}
    \mathbf{x}^i_j = (I_{d}+\tilde{\Lambda}^i)\tilde{\mathbf{x}}^i_{j}+\tilde{\delta}^i. \nonumber
\end{equation}
Here, each $\tilde{\Lambda}^i$ is a random matrix with i.i.d. Gaussian entries according to $\mathcal{N}(0,\frac{\sigma^2}{d})$, and $\tilde{\delta}^i$ is a random Gaussian vector according to $\mathcal{N}(0,{\sigma^2}I_{d})$ where we set $\sigma=0.01$. In  test time, we did not apply any random affine transformation to test samples and instead considered the following three scenarios: (1) no perturbation, (2) adversarial affine distribution shift obtained by optimizing the inner maximization in \eqref{eq:ERM} using projected gradient descent, (3) adversarial perturbations designed by the projected gradient descent algorithm. We used $100$ projected gradient steps with stepsize $0.1$. 

We considered three baselines in the experiments: (1) \texttt{FedAvg} where the server node averages the updated parameters of the local nodes after every gradient step, (2) Distributed FGM training where the nodes perform fast adversarial training \citep{goodfellow2014explaining} by optimizing an $\ell_2$-norm bounded perturbation $\delta^{i}_{j}$ using one gradient step followed by projection onto the ball $\{\delta^{i}_{j}:\, \Vert \delta^{i}_{j} \Vert_2\le \epsilon_{\operatorname{fgm}} \}$, and (3) Distributed PGD training where each node preforms PGD adversarial training \citep{madry2017towards} similar to distributed FGM but uses $10$ projected gradient steps, each followed by projection onto $\{\delta^{i}_{j}:\, \Vert \delta^{i}_{j} \Vert_2\le \epsilon_{\operatorname{pgd}} \}$. We used the value $\epsilon_{\operatorname{fgm}} = \epsilon_{\operatorname{pgd}} = 0.05 \, \mathbb{E}[\Vert\mathbf{x}_i\Vert_2]$ in the experiments. We observed training instability after achieving perfect training accuracy for the baseline \texttt{FedAvg} algorithm, and hence performed early stopping to avoid the instability in the \texttt{FedAvg} experiments. We did not encounter the instability issue in \texttt{FedRobust} experiments.

\begin{figure}
\centering
\includegraphics[width=\linewidth]{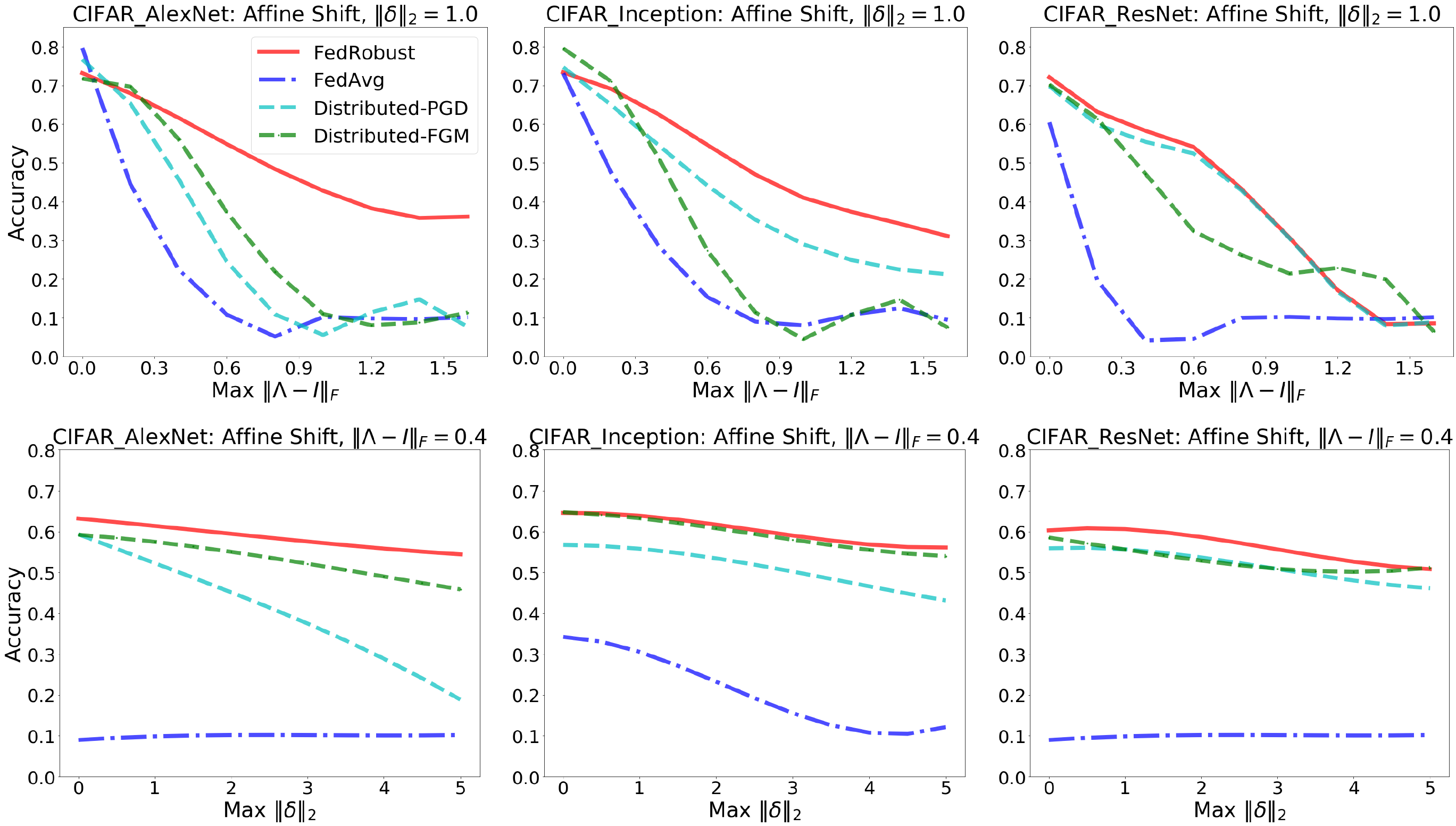}
\caption{Trained networks' test accuracy under affine distribution shifts in the CIFAR-10 experiments. Top row: constraining $\Vert\delta \Vert_2\le 1$ and changing maximum allowed $\Vert \Lambda - I\Vert_F$. Bottom row: constraining $\Vert \Lambda - I\Vert_F\le 0.4$ and changing maximum allowed $\Vert\delta \Vert_2$.} 
\label{Fig: Cifar-affine}
\end{figure}

\subsection{\texttt{FedRobust} vs. \texttt{FedAvg} and adversarial training: Affine distribution shifts}

We tested the performance of the neural net classifiers trained by \texttt{FedRobust}, \texttt{FedAvg}, distributed FGM, and distributed PGD under different levels of affine distribution shifts. Figure \ref{Fig: Cifar-affine} shows the accuracy performance over CIFAR-10 with AlexNet, Inception-Net, and ResNet architectures. 
As demonstrated, \texttt{FedRobust} outperforms the baseline methods in most of the experiments. The improvement over \texttt{FedAvg} can be as large as $54\%$. Moreover, \texttt{FedRobust} improved over distributed FGM and PGD adversarial training, which suggests adversarial perturbations may not be able to capture the complexity of affine distribution shifts. \texttt{FedRobust} also results in $4 \times$ faster training compared to distributed PGD. These improvements motivate \texttt{FedRobust} as a robust and efficient federated learning method to protect against affine distribution shifts. 



\begin{figure}
\centering
\includegraphics[width=\linewidth]{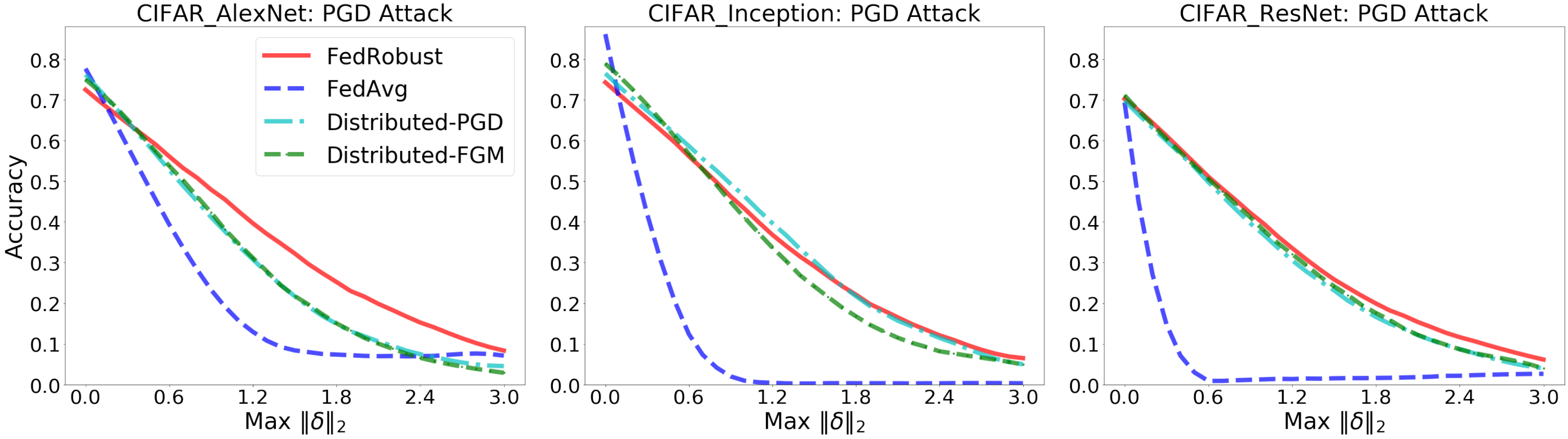}
\caption{Trained networks' test accuracy under PGD perturbations in the CIFAR-10 experiments. $X$-axis shows the maximum allowed $\ell_2$-norm for PGD perturbations.} 
\label{Fig: Cifar-pgd}
\end{figure}

\subsection{\texttt{FedRobust} vs. \texttt{FedAvg} and adversarial training: Adversarial perturbations}

Figure \ref{Fig: Cifar-pgd} summarizes our numerical results of \texttt{FedRobust} and other baselines over CIFAR-10 where the plots show the test accuracy under different levels of $\ell_2$-norm perturbations. While we motivated \texttt{FedRobust} as a federated learning scheme protecting against affine distribution shifts, we empirically observed its robust performance against adversarial perturbations as well. The achieved adversarial robustness in almost all cases matches the robustness offered by distributed FGM and PGD adversarial training. These numerical results indicate that affine distribution shifts can cover the distribution changes caused by norm-bounded adversarial perturbations. 
In summary, our numerical experiments demonstrate the efficiency and robustness of \texttt{FedRobust} against PGD adversarial attacks.

\begin{figure}
\centering
\includegraphics[width=.69\linewidth]{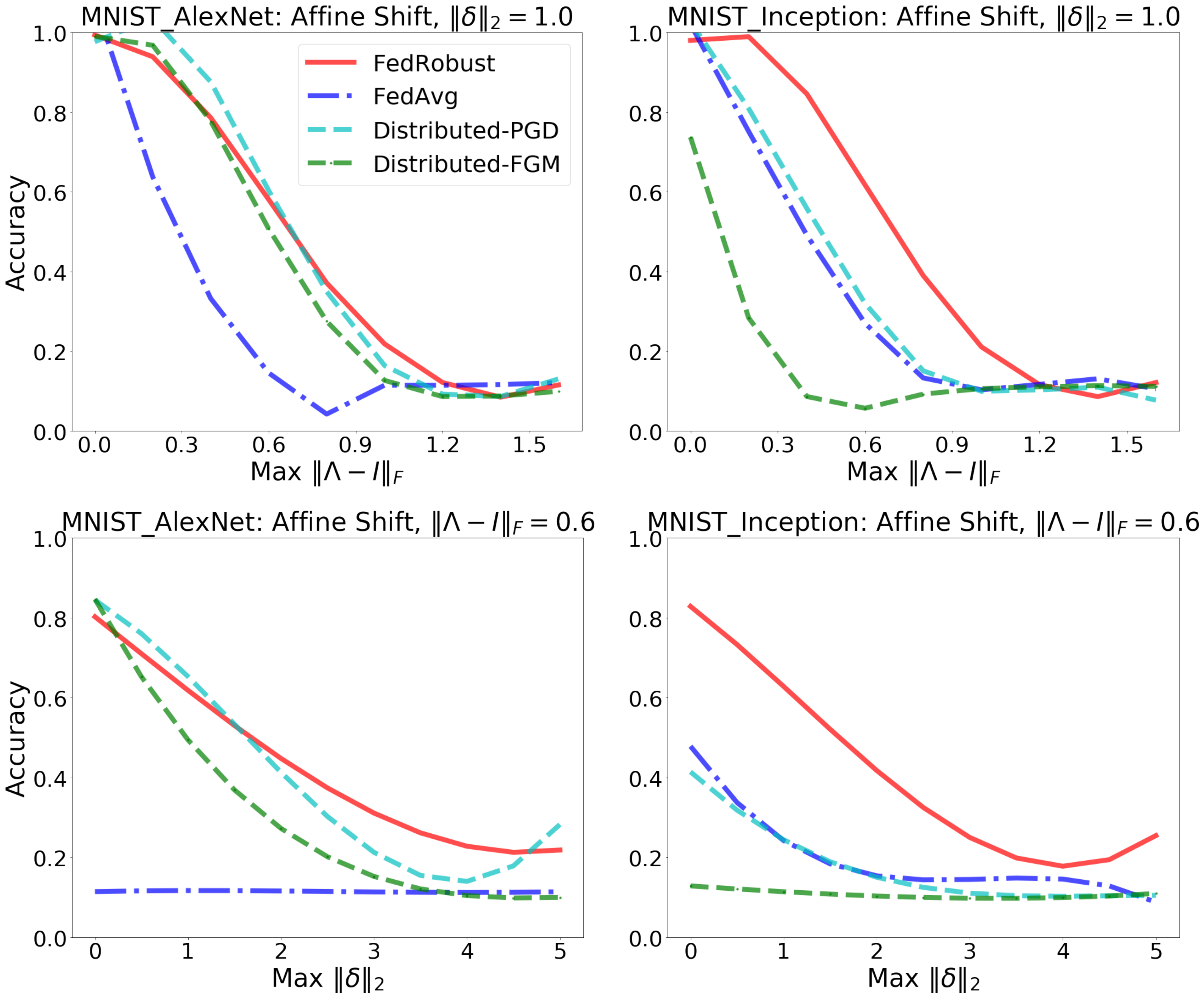}
\caption{Trained networks' test accuracy under affine distribution shifts in the MNIST experiments. Top row: constraining $\Vert\delta \Vert_2\le 1$ and changing maximum allowed $\Vert \Lambda - I\Vert_F$, bottom row: constraining $\Vert \Lambda - I\Vert_F\le 0.6$ and changing maximum allowed $\Vert\delta \Vert_2$.} 
\label{Fig: mnist-affine}
\end{figure}

\subsection{Numerical results for MNIST data}
We repeated the CIFAR experiments in Figures \ref{Fig: Cifar-affine} and \ref{Fig: Cifar-pgd} for the MNIST dataset. Figure \ref{Fig: mnist-affine} shows the numerical results under affine distribution shifts. The figure's top row includes the plots for fixed maximum delta norm $\Vert \delta\Vert_2\le 1$ and different levels of maximum allowed $\Vert \Lambda - I\Vert_F$, while in the bottom row we fix the maximum allowed linear shift $\Vert \Lambda - I\Vert_F\le 0.6$ and evaluate the test accuracy under different levels of $\Vert\delta\Vert_2$. As shown in the plots, \texttt{FedRobust} results in the best performance in most of the evaluations, which indicates the superior performance of \texttt{FedRobust} against affine distribution shifts. Figure \ref{Fig: mnist-pgd} shows the test accuracy of the trained networks under different levels of adversarial PGD perturbations. The figure's experiments again shows that \texttt{FedRobust} can effectively shield against PGD adversarial attacks and achieve a comparable performance to PGD and FGM adversarial training.

\begin{figure}
\centering
\includegraphics[width=.69\linewidth]{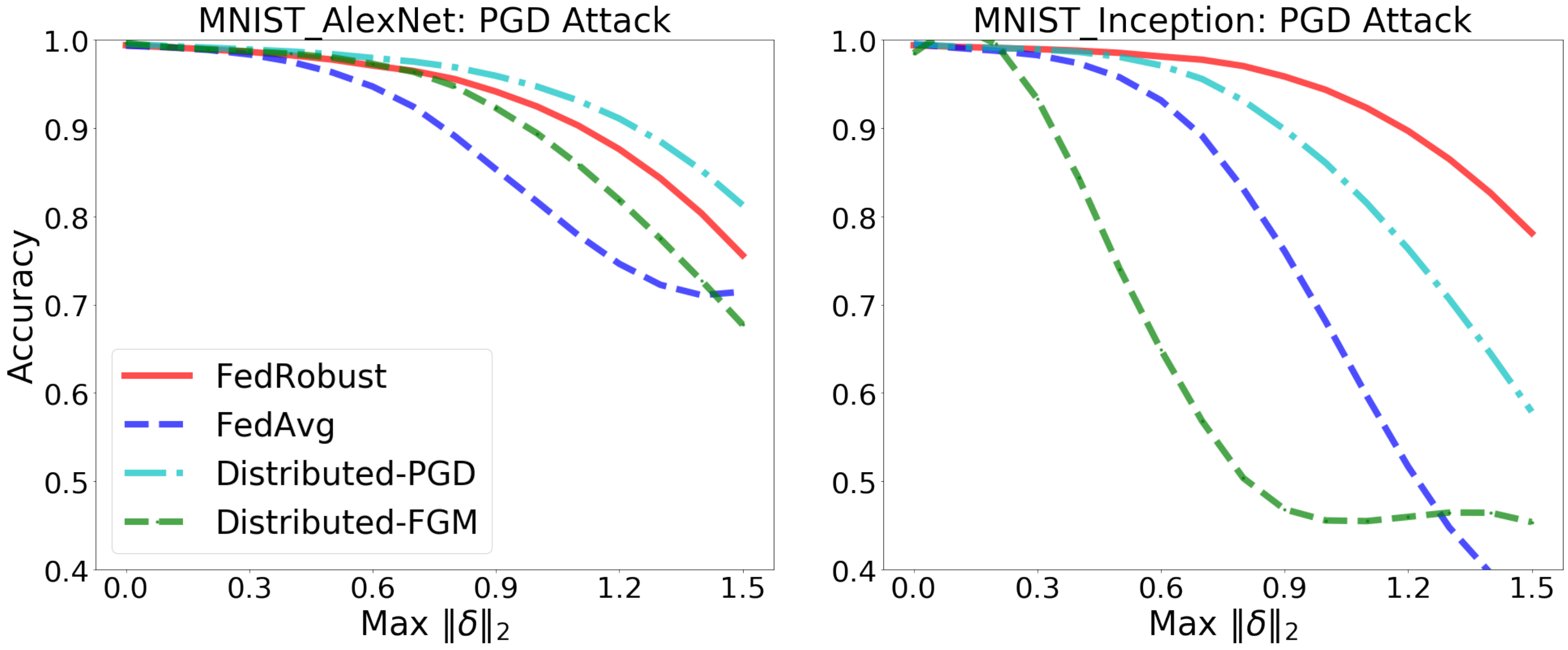}
\caption{Trained networks' test accuracy under PGD perturbations in the MNIST experiments. $X$-axis shows the maximum allowed $\ell_2$-norm for PGD perturbations.} 
\label{Fig: mnist-pgd}
\end{figure}

%% file: 7-opt-proofs-SGDA.tex
\section{Preliminaries and Useful Lemmas}

In this section, we provide preliminary and useful results in order to prove Theorems \ref{Thm: PL-PL convergence} and \ref{Thm: nonconvex-PL convergence}. For notational convenience, we use  the following short-hand notations:
\begin{table}[h!] 
\begin{center}
    \begin{tabular}{ c c }
    \toprule
    Notation & Description  \\
    \midrule
    $\displaystyle \bpsi^i_t = \left( \Lambda^i_t \, , \, \delta^i_t \right)$
    & maximization variables of node $i$ iteration $t$  \\
    $\displaystyle \Psi_t 
    = \left( \bpsi^1_t \,;\, \cdots \, ; \, \bpsi^n_t \right)$
    &
    \begin{tabular}{@{}c@{}}concatenation of all nodes' maximization \\ models at iteration $t$ \end{tabular}  \\
    $\displaystyle \wbar_t = \frac{1}{n} \sum_{i \in [n]} \w^i_t$ 
    & average model at iteration $t$ \vspace{0cm}\\
    $\displaystyle a_t = \E [ \Phi(\wbar_t) ] - \Phi^*$ & \begin{tabular}{@{}c@{}}optimality gap measure \\ between $\Phi(\wbar_t)$ and $\min_{\w} \Phi(\w)$ \end{tabular}  \vspace{0.2cm}\\
    $\displaystyle b_t = \E [ \Phi(\wbar_t) - f(\wbar_t, \Psi_t)]$ & \begin{tabular}{@{}c@{}}optimality gap measure \\ between $f(\wbar_t, \Psi_t)$ and $\max_{\Psi} f(\wbar_t, \Psi)$ \end{tabular} \vspace{0.2cm}\\
    $\displaystyle e_t = \frac{1}{n} \sum_{i \in [n]} \E \norm{\w^i_t - \wbar_{t}}^2 $ & \begin{tabular}{@{}c@{}}average deviation of the local models  \\ from the average model at iteration $t$ \end{tabular} \\
    $\displaystyle g_t 
    = \E \norm{\frac{1}{n} \sum_{i \in [n]} \grw f^i(\w^i_t , \bpsi^i_t)}^2$  & \begin{tabular}{@{}c@{}}norm squared of  \\ local gradients w.r.t $\w$ at iteration $t$ \end{tabular} \\
    $\displaystyle h_t 
    = \E \norm{\gr \Phi(\wbar_{t}) - \frac{1}{n} \sum_{i \in [n]} \grw f^i(\w^i_t , \bpsi^i_t)}^2$ & \begin{tabular}{@{}c@{}}norm squared of deviation in gradients w.r.t $\w$ \\ of $\max_{\Psi} f(\wbar_t, \Psi)$ and local functions $f^i(\w^i_t , \bpsi^i_t)$ \end{tabular}\\
    \bottomrule
    \end{tabular}
    \vspace{1mm}
    \caption{Table of notations.}
    \label{Table: Notations}
\end{center}
\end{table}

Now, we present a set of useful lemmas and observations which we will invoke to prove the convergence results for both PL-PL and nonconvex-PL loss cases. The following lemma establishes the Lipschitz gradient parameter for the global function given those of the local objectives.
\begin{lemma} \label{Lemma: f Lipschitz}
If the local functions $f^i$s have Lipschits gradients with parameters stated in Assumption \ref{assumption:smooth}, then the global function $f$ has also Lipschitz gradients as follows: for any $\w, \w', \Psi, \Psi'$ it holds that
\begin{gather}
    \norm{\grw f(\w, \Psi) \! - \! \grw f(\w', \Psi)} 
    \leq
    L_1 \norm{\w \! - \! \w'}, \,
    \norm{\grw f(\w, \Psi) \! - \! \grw f(\w, \Psi')} 
    \leq
    \frac{L_{12}}{\sqrt{n}} \norm{\Psi \! - \! \Psi'}_F, \\
    \norm{\grPsi f(\w, \Psi) \! - \! \grPsi f(\w', \Psi)}_F 
    \leq 
    \frac{L_{21}}{\sqrt{n}} \norm{\w \! - \! \w'}, \,
    \norm{\grPsi f(\w, \Psi) \! - \! \grPsi f(\w, \Psi')}_F 
    \leq
    \frac{L_2}{n} \norm{\Psi \! - \! \Psi'}_F.
\end{gather}
\end{lemma}
\begin{proof}
We defer the proof to Section \ref{proof Lemma: f Lipschitz}.
\end{proof}
Recall the definition of the function $\Phi(\cdot)$, that is,
\begin{align}
    \Phi(\w) 
    \coloneqq 
    \max_{\Psi} f(\w, \Psi)
    =
    \max_{\bpsi^1, \cdots, \bpsi^n} \frac{1}{n} \sum_{i \in [n]} f^i(\w,\bpsi^i)
    =
    \max_{(\Lambda^1, \delta^1), \cdots, (\Lambda^n, \delta^n)} \frac{1}{n} \sum_{i \in [n]} f^i(\w,\Lambda^i, \delta^i).
\end{align}
Next lemma shows that $\Phi$ has Lipschitz gradients and characterizes its parameter.
\begin{lemma}[\cite{nouiehed2019solving}] \label{lemma:L_Phi}
If Assumptions \ref{assumption:smooth} and \ref{assumption:PL} (ii) hold, that is, the local objectives have Lipschitz gradients and $-f(\w, \cdot)$ is $\mu_2$-PL, then we have
\begin{align}
    \gr \Phi(\w)
    =
    \grw f(\w, \Psi^*(\w)),
\end{align}
where $\Psi^*(\w) \in \argmax_{\Psi} f(\w, \Psi)$ for any $\w$. Moreover, $\Phi$ has Lipschitz gradients with parameter $L_{\Phi} = L_1 + \frac{L_{12} L_{21}}{2n\mu_2}$.
\end{lemma}
\begin{proof}
We defer the proof to Section \ref{proof lemma:L_Phi}.
\end{proof}
Next lemma shows the contraction of the sequence $\{\E [ \Phi(\wbar_t) ]\}_{t \geq 0}$ when running the update rule of \texttt{FedRobust} method in Algorithm \ref{Alg: FedRobust}. Please refer to Table \ref{Table: Notations} to recall the definition of $h_t$ and $g_t$.
\begin{lemma} \label{Lemma: Phi contraction}
If Assumptions \ref{assumption:stoch-gradients} and \ref{assumption:smooth} hold, then the iterates of {\normalfont \texttt{FedRobust}} satisfy the following contraction inequality for any iteration $t \geq 0$
\begin{align}
    \E [ \Phi(\wbar_{t+1}) ] - \E [ \Phi(\wbar_{t}) ]
    \leq
    - \frac{\eta_1}{2} \E \norm{\gr \Phi(\wbar_{t})}^2 
    +
    \frac{\eta_1}{2} h_t
    -
    \frac{\eta_1}{2} \left( 1 - \eta_1 L_{\Phi} \right) g_t
    +
    \eta_1^2 \frac{L_{\Phi}}{2} \frac{\sigmaw}{n}.
\end{align}
\end{lemma}
\begin{proof}
We defer the proof to Section \ref{proof Lemma: Phi contraction}.
\end{proof}
Next lemma further bounds $h_t$ w.r.t. the two sequences $b_t$ and $e_t$.
\begin{lemma} \label{Lemma: h_t bound}
If Assumptions \ref{assumption:smooth} and \ref{assumption:PL} (ii) hold, that is, the local objectives have Lipschitz gradients and $-f(\w, \cdot)$ is $\mu_2$-PL, then we have
\begin{align} \label{eq:h_t-final}
    h_t
    \leq
    \frac{4 L_{12}^2}{\mu_2 n} b_t 
    +
    2 L_1^2 e_t.
\end{align}
\end{lemma}
\begin{proof}
We defer the proof to Section \ref{proof Lemma: h_t bound}.
\end{proof}
Next lemma establishes a contraction bound on the sequence $b_t$.
\begin{lemma} \label{Lemma: b_t contraction}
If Assumptions \ref{assumption:stoch-gradients}, \ref{assumption:smooth} and \ref{assumption:PL} (ii) hold, then the sequence of $\{b_t\}_{t \geq 0}$ generated by the {\normalfont \texttt{FedRobust}} iterations with $\eta_2 \leq 1/L_2$ satisfies the following contraction bound:
\begin{align}  \label{eq:b_t-final}
    b_{t+1} 
    &\leq
    (1 - \mu_2 \eta_2 n) \left( 1 + \eta_1 \frac{4 L_{12}^2}{\mu_2 n} \right) b_t
    +
    \frac{\eta_1}{2} \E \norm{\gr \Phi(\wbar_{t})}^2 
    +
    \frac{\eta_1^2}{2} \left( L_1 + L_{\Phi} + 2 \eta_2 L_{21}^2 \right) g_t \\
    &\quad +
    \left( \eta_1 L_1^2 + \eta_2 L_{21}^2 \right) e_t 
    +
    \frac{\eta_1^2}{2} \left( L_1 + L_{\Phi} + 2 \eta_2 L_{21}^2 \right) \frac{\sigmaw}{n} 
    +
    \frac{\eta_2^2}{2} L_2 \sigmapsi,
\end{align}
where $L_{\Phi}$ is the Lipschitz gradient parameter of the function $\Phi(\cdot)$ characterized in Lemma \ref{lemma:L_Phi}.
\end{lemma}
\begin{proof}
We defer the proof to Section \ref{proof Lemma: b_t contraction}.
\end{proof}
Next lemma bounds $e_t$, that is the average deviation of local parameter models from their average.
\begin{lemma} \label{lemma:e_t g_l bound}
If Assumptions \ref{assumption:bounded-degree}, \ref{assumption:stoch-gradients} and \ref{assumption:smooth} hold and the step-size $\eta_1$ satisfies $32 \eta_1^2 (\tau - 1)^2 L_1^2 \leq 1$, then the sequence $e_t = \frac{1}{n} \sum_{i \in [n]} \E \norm{\w^i_t - \wbar_{t}}^2$ is bounded as follows
\begin{align}
    e_{t}
    &\leq
    16 \eta_1^2 (\tau - 1)^2 \rho^2 
    +
    4 \eta_1^2 (\tau - 1) (n+1) \frac{\sigmaw}{n}
    +
    20 \eta_1^2 (\tau - 1) \sum_{l=t_c+1}^{t-1} g_l,
\end{align}
where $t_c$ denotes the index of the most recent server-worker communication, i.e. $t_c = \floor*{\frac{t}{\tau}} \tau$ and we also denote $\rho^2 \coloneqq 3\rho_f^2 + 6 L_{12}^2 (\epsilon_1^2 + \epsilon_2^2)$.
\end{lemma}
\begin{proof}
We defer the proof to Section \ref{proof lemma:e_t g_l bound}.
\end{proof}
Next generic lemma is adopted form \cite{haddadpour2019convergence}.
\begin{lemma} \label{lemma:elimination} 
Assume that two non-negative sequences $\{P_t\}_{t \geq 0}$ and $\{g_t\}_{t \geq 0}$ satisfy the following inequality for each iteration $t \geq 0$ and some constants $0 < \Upsilon < 1$, $L \geq 0$, $B \geq 0$, and $\Gamma \geq 0$:
\begin{align} 
    P_{t+1} 
    &\leq
    \Upsilon P_t
    -
    \frac{\eta_1}{2} \left( 1 - \eta_1 L \right) g_t  
    +
    \eta_1^2 B \sum_{l=t_c+1}^{t-1} g_l
    +
    \Gamma,
\end{align}
where $t_c = \floor*{\frac{t}{\tau}} \tau$. Then, for each $t \geq 0$ we have
\begin{align} 
    P_{t} 
    &\leq
    \Upsilon^t P_0
    +
    \frac{\Gamma}{1 - \Upsilon},
\end{align}
if $\eta_1$ satisfies the following condition
\begin{align} 
    \eta_1 \left( L + \frac{2B}{\Upsilon^{\tau-1} (1 - \Upsilon)} \right) \leq 1.
\end{align}
\end{lemma}
\begin{proof}
We defer the proof to Section \ref{proof lemma:elimination}.
\end{proof}
Next lemma bounds the overall optimality gap $b_t$ averaged over $T$ iterations.
\begin{lemma} \label{lemma: sum b_t}
If Assumptions \ref{assumption:stoch-gradients}, \ref{assumption:smooth} and \ref{assumption:PL} (ii) hold and the step-sizes satisfy the conditions $\eta_2 \leq 1/{L_2}$ and $\frac{\eta_2}{\eta_1} \geq \frac{8 L_{12}^2}{\mu_2^2 n^2}$, then the average of the sequence  $\{b_t\}_{t = 0}^{T-1}$ generated from the {\normalfont \texttt{FedRobust}} can be bounded as follows:
\begin{align} 
    \frac{1}{T} \sum_{t=0}^{T-1} b_{t}
    &\leq
    \frac{4 L_{2}^2}{\mu_2^2 n^2} \frac{\epsilon_1^2 + \epsilon_2^2}{\eta_2 T}
    +
    \frac{\eta_1}{\eta_2} \frac{1}{\mu_2 n} \frac{1}{T} \sum_{t=0}^{T-1} \E \norm{\gr \Phi(\wbar_{t})}^2 \\
    &\quad +
    \frac{\eta_1^2}{\eta_2} \frac{1}{\mu_2 n} \left( L_1 + L_{\Phi} + 2 \eta_2 L_{21}^2 \right) \frac{1}{T} \sum_{t=0}^{T-1} g_{t}
    +
    \frac{1}{\eta_2} \frac{2}{\mu_2 n} \left( \eta_1 L_1^2 + \eta_2 L_{21}^2 \right) \frac{1}{T} \sum_{t=0}^{T-1} e_{t} \\
    &\quad +
    \frac{\eta_1^2}{\eta_2} \frac{1}{\mu_2 n} \left( L_1 + L_{\Phi} + 2 \eta_2 L_{21}^2 \right) \frac{\sigmaw}{n} 
    +
    \eta_2 \frac{L_2}{\mu_2 n} \sigmapsi,
\end{align}
where $L_{\Phi}$ is the Lipschitz gradient parameter of the function $\Phi(\cdot)$ characterized in Lemma \ref{lemma:L_Phi} and $\epsilon_1, \epsilon_2$ represent the radius of the affine perturbation balls, i.e. $\Vert \Lambda^i - I \Vert \leq \epsilon_1$ and $\Vert \delta^i \Vert \leq \epsilon_2$ for each node $i \in [n]$.
\end{lemma}
\begin{proof}
We defer the proof to Section \ref{proof lemma: sum b_t}.
\end{proof}
Next lemma bounds the averaged local model deviations $e_t$ over $T$ iterations.
\begin{lemma}\label{lemma: sum e_t}
If Assumptions \ref{assumption:bounded-degree}, \ref{assumption:stoch-gradients} and \ref{assumption:smooth} hold and the step-size $\eta_1$ satisfies $32 \eta_1^2 (\tau - 1)^2 L_1^2 \leq 1$, then the average of the sequence $e_t$ over $t=0,\cdots,T-1$ is bounded as follows
\begin{align}
    \frac{1}{T} \sum_{t=0}^{T-1} e_{t}
    &\leq
    20 \eta_1^2 (\tau - 1)^2 \frac{1}{T} \sum_{t=0}^{T-1} g_t
    +
    16 \eta_1^2 (\tau - 1)^2 \rho^2
    +
    8 \eta_1^2 (\tau - 1) (n+1) \frac{\sigmaw}{n}.
\end{align}
\end{lemma}
\begin{proof}
We defer the proof to Section \ref{proof lemma: sum e_t}.
\end{proof}


\section{Proof of Theorem \ref{Thm: PL-PL convergence}} \label{proof: Thm PL-PL}
Having established the key lemmas, now we proceed to prove Theorem \ref{Thm: PL-PL convergence} for any $\beta \leq 1/2$. To show the convergence of the sequence $P_t = a_t + \beta b_t$, we firstly need to establish a contraction inequality on $P_{t+1}$ with respect to $P_t$. We begin by the following bound on the sequence $a_t = \E [ \Phi(\wbar_t) ] - \Phi^*$ which is directly implied from Lemma \ref{Lemma: Phi contraction}:
\begin{align} \label{eq: Theorem PL-PL 1}
    a_{t+1}
    \leq
    a_t
    - \frac{\eta_1}{2} \E \norm{\gr \Phi(\wbar_{t})}^2 
    +
    \frac{\eta_1}{2} h_t
    -
    \frac{\eta_1}{2} \left( 1 - \eta_1 L_{\Phi} \right) g_t
    +
    \eta_1^2 \frac{L_{\Phi}}{2} \frac{\sigmaw}{n}.
\end{align}
Using Lemma \ref{Lemma: h_t bound} that shows $h_t \leq 4 L_{12}^2 b_t / (\mu_2 n) + 2 L_1^2 e_t$, the bound in \eqref{eq: Theorem PL-PL 1} yields that
\begin{align} \label{eq: Theorem PL-PL 2}
    a_{t+1}
    \leq
    a_t
    - \frac{\eta_1}{2} \E \norm{\gr \Phi(\wbar_{t})}^2 
    +
    \eta_1 \frac{2 L_{12}^2}{\mu_2 n} b_t
    +
    \eta_1 L_1^2 e_t
    -
    \frac{\eta_1}{2} \left( 1 - \eta_1 L_{\Phi} \right) g_t
    +
    \eta_1^2 \frac{L_{\Phi}}{2} \frac{\sigmaw}{n}.
\end{align}
Next, we employ the result of Lemma \ref{Lemma: b_t contraction} which establishes a contraction bound on the $b_t$ sequence. Putting together with \eqref{eq: Theorem PL-PL 2} implies that
\begin{align} \label{eq: Theorem PL-PL 3}
    P_{t+1} 
    &=
    a_{t+1} + \beta b_{t+1} \\
    &\leq
    a_t
    -
    \frac{\eta_1}{2} \left( 1 - \beta \right) \E \norm{\gr \Phi(\wbar_{t})}^2 \\
    &\quad +
    \beta \left( \eta_1 \frac{2 L_{12}^2}{\beta \mu_2 n} + (1 - \mu_2 \eta_2 n) \left( 1 + \eta_1 \frac{4 L_{12}^2}{\mu_2 n} \right)\right) b_t \\
    &\quad -
    \left( \frac{\eta_1}{2} \left( 1 - \eta_1 L_{\Phi} \right) - \eta_1^2 \frac{\beta}{2} \left( L_1 + L_{\Phi} + 2 \eta_2 L_{21}^2 \right) \right) g_t \\
    &\quad +
    \left( \eta_1 L_1^2 + \beta \left( \eta_1 L_1^2 + \eta_2 L_{21}^2 \right) \right) e_t \\
    &\quad +
    \frac{\eta_1^2}{2} \left( L_{\Phi} + \beta \left( L_1 + L_{\Phi} + 2 \eta_2 L_{21}^2 \right) \right) \frac{\sigmaw}{n}
    +
    \eta_2^2 L_2 \frac{\beta}{2} \sigmapsi.
\end{align}
We begin simplifying the above bound by first considering the first two terms in RHS of \eqref{eq: Theorem PL-PL 3}. We can show that the function $\Phi(\cdot)$ is $\mu_1$-PL \citep{yang2020global}, which implies that
\begin{align} 
    \E \norm{\gr \Phi(\wbar_{t})}^2 
    \geq
    2\mu_1 \E [\Phi(\wbar_t) ] - \Phi^* = 2 \mu_1 a_t.
\end{align}
Therefore, for any $\beta \leq 1/2$ we have 
\begin{align}
    a_t
    -
    \frac{\eta_1}{2} \left( 1 - \beta \right) \E \norm{\gr \Phi(\wbar_{t})}^2
    \leq
    \left( 1 - \frac{1}{2} \mu_1 \eta_1 \right) a_t,
\end{align}
which implies the coefficient of $a_t$ in \eqref{eq: Theorem PL-PL 3} is bounded by $1 - \frac{1}{2} \mu_1 \eta_1$. Next, the coefficient of $\beta b_t$ in \eqref{eq: Theorem PL-PL 3} can be bounded as follows:
\begin{align} 
    \eta_1 \frac{2 L_{12}^2}{\beta \mu_2 n} + (1 - \mu_2 \eta_2 n) \left( 1 + \eta_1 \frac{4 L_{12}^2}{\mu_2 n} \right)
    &=
    1 - \eta_1 \frac{L_1 L_2}{\mu_2 n} \left( \frac{\mu^2_2 \eta_2 n}{\eta_1 L_1 L_2} - \frac{2 L_{21}^2}{\beta L_1 L_2} - 4 (1 - \mu_2 \eta_2 n) \frac{L_{21}^2}{L_1 L_2} \right) \\
    &\stackrel{(a)}{\leq}
    1 - \eta_1 \frac{L_1 L_2}{\mu_2 n} \\
    &\stackrel{(b)}{\leq}
    1 - \frac{1}{2} \mu_1 \eta_1,
\end{align}
where $(a)$ holds for our choice of $\beta$ and assuming $\frac{\mu^2_2 \eta_2 n}{\eta_1 L_1 L_2} \geq 1 + (4+\frac{2}{\beta})\frac{L_{12}^2}{L_1 L_2}$ and $(b)$ is implies from the fact that 
\begin{align} 
    \frac{\eta_1 \frac{L_1 L_2}{\mu_2 n}}{\frac{1}{2} \mu_1 \eta_1}
    =
    2 \left( \frac{L_1}{\mu_1} \right) \left( \frac{L_2}{\mu_2 n} \right)
    \geq 1.
\end{align}
Now that we have bounded the coefficients of $a_t$ and $\beta b_t$ in \eqref{eq: Theorem PL-PL 3}, rearranging the terms and using the assumption $\eta_2 \leq 1/{L_2}$ simplifies the contraction on $P_t$ as follows
\begin{align} \label{eq: Theorem PL-PL 4}
    P_{t+1} 
    &\leq
    \left( 1 - \frac{1}{2} \mu_1 \eta_1 \right) P_t 
    -
    \frac{\eta_1}{2} \left( 1 - \eta_1 \hat{L}_{\beta} \right) g_t 
    +
    \tilde{L}_{\beta} e_t 
    +
    \eta_1^2 \frac{\hat{L}_{\beta}}{2} \frac{\sigmaw}{n}
    +
    \eta_2^2 \frac{L_2}{2} \beta \sigmapsi,
\end{align}
where we picked the following notations for convenient of the exposition
\begin{gather}
    \tilde{L}_{\beta} = (1 + \beta) \eta_1 L_1^2 + \beta  \eta_2 L_{21}^2, \quad
    \hat{L}_{\beta} =  (1 + \beta) L_{\Phi} + \beta L_1 + 2 \beta \frac{L_{21}^2}{L_2}. 
\end{gather}
Next, we use Lemma \ref{lemma:e_t g_l bound} which for $32 \eta_1^2 (\tau - 1)^2 L_1^2 \leq 1$ provides an upper bound on $e_t$ with respect to $g_t$. We can write
\begin{align} \label{eq: Theorem PL-PL 5}
    P_{t+1} 
    &\leq
    \left( 1 - \frac{1}{2} \mu_1 \eta_1 \right) P_t 
    -
    \frac{\eta_1}{2} \left( 1 - \eta_1 \hat{L}_{\beta} \right) g_t 
    +
    20 \eta_1^2 \tilde{L}_{\beta} (\tau - 1) \sum_{l=t_c+1}^{t-1} g_l \\
    &\quad +
    16 \eta_1^2 \tilde{L}_{\beta} (\tau - 1)^2 \rho^2 
    +
    4 \eta_1^2 \tilde{L}_{\beta} (\tau - 1) (n+1) \frac{\sigmaw}{n}
    +
    \eta_1^2 \frac{\hat{L}_{\beta}}{2} \frac{\sigmaw}{n}
    +
    \eta_2^2 \frac{L_2}{2} \beta \sigmapsi.
\end{align}
We have shown in Lemma \ref{lemma:elimination} that how a such contraction sequence converges. In particular, let us pick the following notations and apply the result of Lemma \ref{lemma:elimination} to contraction in \eqref{eq: Theorem PL-PL 5} 
\begin{align}
    L
    &=
    \hat{L}_{\beta}, \\
    \Upsilon 
    &=
    1 - \frac{1}{2} \mu_1 \eta_1, \\
    B
    &=
    20 \tilde{L}_{\beta} (\tau - 1),  \\
    \Gamma
    &=
    16 \eta_1^2 \tilde{L}_{\beta} (\tau - 1)^2 \rho^2
    +
    4 \eta_1^2 \tilde{L}_{\beta} (\tau - 1) (n+1) \frac{\sigmaw}{n}
    +
    \eta_1^2 \frac{\hat{L}_{\beta}}{2} \frac{\sigmaw}{n}
    +
    \eta_2^2 \frac{L_2}{2} \beta \sigmapsi.
\end{align}
It implies that if the step-sizes satisfy the following condition
\begin{gather} \label{eq: condition eta_1 1}
    \eta_1 \left( \hat{L}_{\beta}  + \frac{80  \tilde{L}_{\beta} (\tau - 1)}{\eta_1 \mu_1 \left( 1 - \frac{1}{2} \mu_1 \eta_1 \right)^{\tau-1}} \right) 
    \leq
    1,
\end{gather}
then we have 
\begin{align}
    P_{t} 
    &\leq
    \left( 1 - \frac{1}{2} \mu_1 \eta_1 \right)^t P_0
    +
    32 \eta_1 \frac{\tilde{L}_{\beta}}{\mu_1} (\tau - 1)^2 \rho^2
    +
    8 \eta_1 \frac{\tilde{L}_{\beta}}{\mu_1}  (\tau - 1) (n + 1) \frac{\sigmaw}{n}
    +
    \eta_1 \frac{\hat{L}_{\beta}}{\mu_1} \frac{\sigmaw}{n}
    +
    \frac{\eta_2^2}{\eta_1} \frac{L_2}{\mu_1} \beta \sigmapsi,
\end{align}
which concludes the proof of Theorem \ref{Thm: PL-PL convergence}. Note to hold this result, in addition to condition \eqref{eq: condition eta_1 1}, we have assumed the following constraints on the step-sizes as well 
\begin{gather} 
    \eta_2 L_2 \leq 1, \quad
    32 \eta_1^2 (\tau - 1)^2 L_1^2 
    \leq
    1, \quad
    \frac{\mu^2_2 \eta_2 n}{\eta_1 L_1 L_2}
    \geq
    1 + \left( 4 + \frac{2}{\beta} \right) \frac{L_{12}^2}{L_1 L_2}.
\end{gather}


\section{Proof of Theorem \ref{Thm: nonconvex-PL convergence}} \label{appendix: proof of Thm nonconvex-PL}
We begin the proof by combining the results of Lemmas \ref{Lemma: Phi contraction} and \ref{Lemma: h_t bound} which yields that for every iteration $t = 0, \cdots, T-1$ we have
\begin{align} \label{eq: Theorem nonconvex-PL 1}
    \E \Phi(\wbar_{t+1}) - \E \Phi(\wbar_t) 
    &\leq
    - 
    \frac{\eta_1}{2} \E \norm{\gr \Phi(\wbar_t)}^2
    -
    \frac{\eta_1}{2} \left( 1 - \eta_1 L_{\Phi} \right) g_t
    +
    \eta_1 \frac{2 L_{12}^2}{\mu_2 n} b_t 
    +
    \eta_1 L_1^2 e_t
    +
    \eta_1^2 \frac{L_{\Phi}}{2} \frac{\sigmaw}{n}.
\end{align}
Summing up all the $T$ inequalities in \eqref{eq: Theorem nonconvex-PL 1} for $t = 0, \cdots, T-1$ and dividing by $T$ yields the following
\begin{align}
    \frac{1}{T} \left( \E \Phi(\wbar_{T}) - \Phi(\wbar_0) \right)
    &\leq
    - 
    \frac{\eta_1}{2} \frac{1}{T} \sum_{t=0}^{T-1} \E \norm{\gr \Phi(\wbar_{t})}^2 \\
    &\quad -
    \frac{\eta_1}{2} \left(1 - \eta_1 L_{\Phi} \right) \frac{1}{T} \sum_{t=0}^{T-1} g_{t}\\
    &\quad +
    \eta_1 \frac{2 L_{12}^2}{\mu_2 n} \frac{1}{T} \sum_{t=0}^{T-1} b_t \\
    &\quad +
    \eta_1 L_1^2 \frac{1}{T} \sum_{t=0}^{T-1} e_{t} \\
    &\quad +
    \eta_1^2 \frac{L_{\Phi}}{2} \frac{\sigmaw}{n}.
\end{align}
Next we use Lemmas \ref{lemma: sum b_t} and then Lemma \ref{lemma: sum e_t} to replace the terms $\frac{1}{T} \sum_{t=0}^{T-1} b_{t}$ and $\frac{1}{T} \sum_{t=0}^{T-1} e_{t}$ and rewrite the above bound in terms of $\frac{1}{T} \sum_{t=0}^{T-1} g_{t}$. It yields that
\begin{align}
    \frac{1}{T} \left( \E \Phi(\wbar_{T}) - \Phi(\wbar_0) \right)
    &\leq
    - 
    \frac{\eta_1}{2} \left( 1 - \eta_1 \frac{4 L_{12}^2 L_2}{\mu^2_2 n^2} \right) \frac{1}{T} \sum_{t=0}^{T-1} \norm{\gr \Phi(\wbar_{t})}^2 \\
    &\quad -
    \frac{\eta_1}{2} \left(1 - \eta_1 \left( \hat{L}  + 40 \tilde{L} (\tau - 1)^2 \right) \right) \frac{1}{T} \sum_{t=0}^{T-1} g_{t}\\
    &\quad +
    \frac{\eta_1}{\eta_2} \frac{8 L_{12}^2 L_2^2}{\mu_2^3 n^3}\frac{ \epsilon_1^2 + \epsilon_2^2}{T}
    +
    16 \eta_1^2 \tilde{L} (\tau - 1)^2 \rho^2
    +
    \frac{\eta_1^2}{2} \hat{L} \frac{\sigmaw}{n}
    +
    \eta_1 \eta_2 \frac{4 L_{12}^2}{\mu_2^2 n^2} \hat{L} \sigmapsi,
\end{align}
where we adopt the following short-hand notations
\begin{gather}
    \tilde{L} = \frac{3}{2} \eta_1 L_1^2 + \frac{1}{2}  \eta_2 L_{21}^2, \quad
    \hat{L} =  \frac{3}{2} L_{\Phi} + \frac{1}{2} L_1 + \frac{L_{21}^2}{L_2}. 
\end{gather}
Finally, we use the assumption $\eta_1 ( \hat{L}  + 40 \tilde{L} (\tau - 1)^2 ) \leq 1$ to remove the term $\frac{1}{T} \sum_{t=0}^{T-1} g_{t}$ and apply $\frac{\eta_1}{\eta_2} \leq \frac{\mu_2^2 n^2}{8 L_{12}^2}$ to simply the bound and conclude the proof:
\begin{align} 
    \frac{1}{T} \sum_{t=0}^{T-1} \E \norm{\gr \Phi(\wbar_{t})}^2
    &\leq
    \frac{4 \Delta_{\Phi}}{\eta_1 T}  
    +
    \frac{4 L_2^2}{\mu_2^2 n^2}\frac{ \epsilon_1^2 + \epsilon_2^2}{\eta_1 T} 
    +
    64 \eta_1 \tilde{L} (\tau - 1)^2 \rho^2 \\
    &\quad +
    16 \eta_1 \tilde{L}  (\tau - 1) (n + 1) \frac{\sigmaw}{n}
    +
    2 \eta_1 \hat{L} \frac{\sigmaw}{n}
    +
    \frac{\eta_2^2}{\eta_1} L_2 \sigmapsi.
\end{align}


\section{Proof of Useful Lemmas} 

\subsection{Proof of Lemma \ref{Lemma: f Lipschitz}} \label{proof Lemma: f Lipschitz}
Proof of all four cases in the claim is simple. We derive the proof for the fourth one as an instance. Recall definition of the global function $f$, that is
\begin{align}
    f(\w, \Psi)
    =
    \frac{1}{n} \sum_{i \in [n]} f^i(\w,\bpsi^i).
\end{align}
Therefore, the gradient of $f$ with respect to $\Psi$ is
\begin{align}
    \grPsi f(\w, \Psi)
    =
    \begin{pmatrix}
        \frac{\partial}{\partial \bpsi^1} f(\w, \Psi)\\
        \vdots \\
        \frac{\partial}{\partial \bpsi^n} f(\w, \Psi)
    \end{pmatrix}
    =
    \frac{1}{n}
    \begin{pmatrix}
        \grpsi f^1(\w, \bpsi^1)\\
        \vdots \\
        \grpsi f^n(\w, \bpsi^n)
    \end{pmatrix}.
\end{align}
We can then write for any $\w, \Psi=(\bpsi^1; \cdots; \bpsi^n), \Psi'=({\bpsi'}^1; \cdots;{\bpsi'}^n)$ and using Assumption \ref{assumption:smooth} that
\begin{align}
    \norm{\grPsi f(\w, \Psi) \! - \! \grPsi f(\w, \Psi')}^2_F 
    &=
    \frac{1}{n^2} \sum_{i \in [n]} \norm{\grpsi f^i(\w, \bpsi^i) - \grpsi f^i(\w, {\bpsi'}^i)}^2_F \\
    &\leq
    \frac{L_2^2}{n^2} \sum_{i \in [n]} \norm{\bpsi^i - {\bpsi'}^i}^2_F \\
    &=
    \frac{L_2^2}{n^2} \norm{\Psi \! - \! \Psi'}^2_F.
\end{align}

\subsection{Proof of Lemma \ref{lemma:L_Phi}} \label{proof lemma:L_Phi}
The detailed proof can be found in \cite{nouiehed2019solving}, Lemma A.5. Note that in our case, according to Lemma \ref{Lemma: f Lipschitz} the function $f$ has Lipschitz gradients with constants $L_1, L_{12}/\sqrt{n}, L_{21}/\sqrt{n}, L_2/n$; implying the Lipschitz gradient parameter of the function $\Phi$ to be
\begin{align}
    L_{\Phi} = L_1 + \frac{(L_{12}/\sqrt{n}) (L_{21}/\sqrt{n})}{2 \mu_2} = L_1 + \frac{L_{12} L_{21}}{2n \mu_2}.
\end{align}

\subsection{Proof of Lemma \ref{Lemma: Phi contraction}} \label{proof Lemma: Phi contraction}
We invoke Lemma \ref{lemma:L_Phi} which shows that the gradient of the function $\Phi(\cdot)$ is $L_{\Phi}$-Lipschitz. We can write
\begin{align} \label{eq: Phi contraction 1}
    \Phi(\wbar_{t+1}) - \Phi(\wbar_{t})
    &\leq
    \left \langle \gr \Phi(\wbar_{t}) , \wbar_{t+1} - \wbar_{t} \right \rangle
    +
    \frac{L_{\Phi}}{2} \norm{\wbar_{t+1} - \wbar_{t}}^2 \\
    &=
    -\eta_1 \left \langle \gr \Phi(\wbar_{t}), \frac{1}{n} \sum_{i \in [n]} \sgrw f^i(w^i_t , \bpsi^i_t) \right \rangle
    +
    \eta_1^2 \frac{L_{\Phi}}{2} \norm{\frac{1}{n} \sum_{i \in [n]} \sgrw f^i(w^i_t , \bpsi^i_t)}^2 ,
\end{align}
where we use the update rule of \texttt{FedRobust} and note that the difference of averaged models can be written as $\wbar_{t+1} - \wbar_{t} = -\eta_1 \frac{1}{n} \sum_{i \in [n]} \sgrw f^i(w^i_t , \bpsi^i_t)$. Moreover, since the stochastic gradients $\sgrw f^i$ are unbiased and variance-bounded by $\sigmaw$, we can take expectation from both sides of \eqref{eq: Phi contraction 1} and further simplify it as follows
\begin{align}
    \E [\Phi(\wbar_{t+1}) - \E [\Phi(\wbar_{t})]
    \leq
    - \frac{\eta_1}{2} \E \norm{\gr \Phi(\wbar_{t})}^2 
    +
    \frac{\eta_1}{2} h_t
    -
    \frac{\eta_1}{2} \left( 1 - \eta_1 L_{\Phi} \right) g_t
    +
    \eta_1^2 \frac{L_{\Phi}}{2} \frac{\sigmaw}{n}.
\end{align}
In above, we used the inequality $ 2 \langle \bba, \bbb \rangle = \Vert \bba \Vert^2 + \Vert \bbb \Vert^2 - \Vert \bba - \bbb \Vert^2$ as well as the notations for $g_t$ and $h_t$ as defined in Table \ref{Table: Notations}.

\subsection{Proof of Lemma \ref{Lemma: h_t bound}} \label{proof Lemma: h_t bound}
We begin bounding $h_t$ by adding/subtracting the term $\grw f(\wbar_t, \Psi_t)$ and use the inequality $\Vert \bba + \bbb \Vert^2 \leq 2\Vert \bba \Vert^2 + 2\Vert \bbb \Vert^2$ to write
\begin{align} \label{eq: h_t 1}
    h_t
    &=
    \E \norm{\gr \Phi(\wbar_{t}) - \frac{1}{n} \sum_{i \in [n]} \grw f^i(\w^i_t , \bpsi^i_t)}^2 \\
    &\leq
    2 \E \norm{\gr \Phi(\wbar_{t}) - \grw f(\wbar_t, \Psi_t)}^2
    +
    2 \E \norm{\grw f(\wbar_t, \Psi_t) - \frac{1}{n} \sum_{i \in [n]} \grw f^i(\w^i_t , \bpsi^i_t)}^2.
\end{align}
The first term in RHS of \eqref{eq: h_t 1} can be bounded as follows:
\begin{align}  \label{eq: h_t 2}
    \E \norm{\gr \Phi(\wbar_t) - \grw f(\wbar_t, \Psi_t)}^2 
    &=
    \E \norm{\grw f(\wbar_t, \Psi^*(\wbar_t)) - \grw f(\wbar_t, \Psi_t)}^2 \\
    &\stackrel{(a)}{\leq}
    \frac{L_{12}^2}{n} \E \norm{\Psi^*(\wbar_t) - \Psi_t}^2_F \\
    &\stackrel{(b)}{\leq}
    \frac{2 L_{12}^2}{\mu_2 n} \E \left[ \Phi(\wbar_{t}) - f(\wbar_{t}, \Psi_{t}) \right] \\
    &\stackrel{(c)}{=}
    \frac{2 L_{12}^2}{\mu_2 n} b_t.
\end{align}
In above and to derive $(a)$, we employ the result of Lemma \ref{Lemma: f Lipschitz} which shows that given Assumption \ref{assumption:smooth}, the gradient function $\grw f(\w, \cdot)$ is $L_{12}/\sqrt{n}$ Lipschitz. To derive $(b)$, we use Assumption \ref{assumption:PL} (ii) and lastly, $(c)$ is implied from the definition of $b_t$. The second term in RHS of \eqref{eq: h_t 1} can be bounded by noting that the local gradients $\grw f^i(\cdot, \bpsi^i)$ are $L_1$-Lipschitz, which we can write
\begin{align} \label{eq: h_t 3}
    \E \norm{\grw f(\wbar_t, \Psi_t) - \frac{1}{n} \sum_{i \in [n]} \grw f^i(\w^i_t , \bpsi^i_t)}^2 
    &=
    \E \norm{\frac{1}{n} \sum_{i \in [n]} \grw f^i(\wbar_t , \bpsi^i_t) - \frac{1}{n} \sum_{i \in [n]} \grw f^i(\w^i_t , \bpsi^i_t)}^2 \\
    &\leq
    \frac{L_1^2}{n} \sum_{i \in [n]} \E \norm{\w^i_t - \wbar_{t}}^2 \\
    &=
    L_1^2 e_t.
\end{align}
Finally, plugging \eqref{eq: h_t 2} and \eqref{eq: h_t 3} back in \eqref{eq: h_t 1} implies the claim of the lemma, that is
\begin{align} \label{eq:h_t-final}
    h_t
    \leq
    \frac{4 L_{12}^2}{\mu_2 n} b_t 
    +
    2 L_1^2 e_t.
\end{align}

\subsection{Proof of Lemma \ref{Lemma: b_t contraction}} \label{proof Lemma: b_t contraction}
We begin the proof by noting the definition of $b_t$ and use the fact that the gradients $\grPsi f(\w, \cdot)$ are $\frac{L_2}{n}$-Lipschitz (Refer to Lemma \ref{Lemma: f Lipschitz}). We can accordingly write
\begin{align} \label{eq:b_t bound 1}
    \Phi(\wbar_{t+1}) - f(\wbar_{t+1}, \Psi_{t+1}) 
    &\leq
    \Phi(\wbar_{t+1}) - f(\wbar_{t+1}, \Psi_{t})
    -
    \langle \gr_{\Psi} f(\wbar_{t+1}, \Psi_{t}) , \Psi_{t+1} - \Psi_{t} \rangle \\
    &\quad +
    \frac{L_2}{2n} \norm{\Psi_{t+1} - \Psi_{t}}^2_F.
\end{align}
In this work, we define the inner product for any two matrices $A,B$ as follows
\begin{align} 
    \langle A , B \rangle \coloneqq \operatorname{Tr} (A^\top B).
\end{align}
Note that according to the ascent update rule of \texttt{FedRobust} in Algorithm \ref{Alg: FedRobust}, we can write 
\begin{align} 
    \Psi_{t+1} - \Psi_{t} = \eta_2 \tilde{\partial}_t f,
\end{align}
where we adopt the following short-hand notation for the stochastic gradients at iteration $t$ with respect to the maximization variables $\bpsi^i_t = (\Lambda^i_t, \delta^i_t)$
\begin{align}
    \tilde{\partial}_t f
    =
    \begin{pmatrix}
        \sgrpsi f^1(\w^1_t, \bpsi^1_t)\\
        \vdots \\
        \sgrpsi f^n(\w^n_t, \bpsi^n_t)
    \end{pmatrix}
    =
    \begin{pmatrix}
        \sgr_{\Lambda} f^1(\w^1_t, \Lambda^1_t, \delta^1_t) 
        &
        \sgr_{\delta} f^1(\w^1_t, \Lambda^1_t, \delta^1_t)\\
        \vdots & \vdots\\
        \sgr_{\Lambda} f^n(\w^n_t, \Lambda^n_t, \delta^n_t) 
        &
        \sgr_{\delta} f^n(\w^n_t, \Lambda^n_t, \delta^n_t)
    \end{pmatrix}.
\end{align}
We also denote the gradients by $\partial_t f = \E [\tilde{\partial}_t f]$ where the expectation is with respect to the randomness in stochastic gradients $\sgrpsi f^i$. According to Assumption \ref{assumption:stoch-gradients}, each of the local stochastic gradients $\sgrpsi f^i(\w^i_t, \bpsi^i_t)$ are variance-bounded by $\sigmapsi$. Therefore, we can bound the variance of $\tilde{\partial}_t f$ as $\E \Vert \tilde{\partial}_t f - \partial_t f \Vert_F^2 \leq n \sigmapsi$. Now, we can plug these back in \eqref{eq:b_t bound 1} which implies
\begin{align} \label{eq:b_t bound 2}
    \Phi(\wbar_{t+1}) - \E f(\wbar_{t+1}, \Psi_{t+1})
    &\leq
    \Phi(\wbar_{t+1}) - f(\wbar_{t+1}, \Psi_{t})
    -
    \eta_2 \frac{n}{2} \norm{\gr_{\Psi} f(\wbar_{t+1}, \Psi_{t})}^2_F + \eta_2^2 \frac{L_2}{2} \sigmapsi\\
    &\quad +
    \eta_2 \frac{n}{2} \norm{\gr_{\Psi} f(\wbar_{t+1}, \Psi_{t}) - \frac{1}{n} \partial_t f}^2_F
    -
    \frac{\eta_2}{2n} \left( 1 - \eta_2 L_2 \right)
    \norm{\partial_t f}^2_F,
\end{align}
where the expectation is with respect to the randomness of the stochastic gradients $\tilde{\partial}_t f$ while conditioning on all the randomness history. Now recall from Assumption \ref{assumption:PL} (ii) that $- f(\wbar_{t+1}, \cdot)$ is $\mu_2$-PL implying that $\Vert \gr_{\Psi} f(\wbar_{t+1}, \Psi_{t}) \Vert^2_F \geq 2 \mu_2 (\Phi(\wbar_{t+1}) - f(\wbar_{t+1}, \Psi_{t}))$. Moreover, assume that $\eta_2 \leq 1/L_2$ to remove the last term in \eqref{eq:b_t bound 2}. Putting altogether implies that
\begin{align} \label{eq:b_t bound 3}
    \Phi(\wbar_{t+1}) - \E f(\wbar_{t+1}, \Psi_{t+1})
    &\leq
    (1 - \mu_2 \eta_2 n) \left( \Phi(\wbar_{t+1}) - f(\wbar_{t+1}, \Psi_{t}) \right)
    +
    \eta_2^2 \frac{L_2}{2} \sigmapsi\\
    &\quad +
    \eta_2 \frac{n}{2} \norm{\gr_{\Psi} f(\wbar_{t+1}, \Psi_{t}) - \frac{1}{n} \partial_t f}^2_F.
\end{align}
Next, we continue to bound the last term in RHS of \eqref{eq:b_t bound 3}. We can write
\begin{align} \label{eq:b_t bound 4}
    \norm{\gr_{\Psi} f(\wbar_{t+1}, \Psi_{t}) - \frac{1}{n} \partial_t f}^2_F 
    &=
    \frac{1}{n^2} \sum_{i \in [n]} \norm{\grpsi f^i(\wbar_{t+1}, \bpsi^i_t) - \grpsi f^i(\w^i_t, \bpsi^i_t)}^2_F \\
    &\leq
    \frac{L_{21}^2}{n^2} \sum_{i \in [n]} \norm{\wbar_{t+1} - \w^i_t}^2 \\
    &\leq
    \frac{2 L_{21}^2}{n^2} \sum_{i \in [n]} \norm{\w^i_t - \wbar_{t}}^2
    +
    \frac{2 L_{21}^2}{n} \norm{\wbar_{t+1} - \wbar_{t}}^2,
\end{align}
where the first inequality above uses Assumption \ref{assumption:smooth} on Lipschitz continuity of local gradients and the second inequality simply uses the inequality $\Vert \bba + \bbb \Vert^2 \leq 2\Vert \bba \Vert^2 + 2\Vert \bbb \Vert^2$. Next, let us bound the term $\Vert \wbar_{t+1} - \wbar_{t} \Vert^2$ in expectation as follows. Using the descent update rule in Algorithm \ref{Alg: FedRobust} and considering Assumption \ref{assumption:stoch-gradients} on variance of the stochastic gradients $\sgrw f^i$ we can write
\begin{align} \label{eq:b_t bound 5}
    \E \norm{\wbar_{t+1} - \wbar_{t}}^2
    &=
    \eta_1^2 \E \norm{\frac{1}{n} \sum_{i \in [n]} \sgrw f^i(\w^i_t , \bpsi^i_t)}^2 \\
    &\leq
    \eta_1^2 \E \norm{\frac{1}{n} \sum_{i \in [n]} \grw f^i(\w^i_t , \bpsi^i_t)}^2    
    +
    \eta_1^2 \frac{\sigmaw}{n} \\
    &=
    \eta_1^2 g_t + \eta_1^2 \frac{\sigmaw}{n},
\end{align}
where we use the short-hand notation of $g_t$ also listed in Table \ref{Table: Notations}. Plugging \eqref{eq:b_t bound 5} back in \eqref{eq:b_t bound 4} and noting the notation $e_t = \frac{1}{n} \sum_{i \in [n]} \E \norm{\w^i_t - \wbar_{t}}^2$ implies that
\begin{align} \label{eq:b_t bound 6}
    \E \norm{\gr_{\Psi} f(\wbar_{t+1}, \Psi_{t}) - \frac{1}{n} \partial_t f}^2_F 
    &\leq
    \frac{2 L_{21}^2}{n} e_t
    +
    \eta_1^2 \frac{2 L_{21}^2}{n} g_t
    +
    \eta_1^2 \frac{2 L_{21}^2}{n} \frac{\sigmaw}{n}.
\end{align}
Before proceeding to bound more terms, let us recall what we have shown till this point. We plug \eqref{eq:b_t bound 6} back in \eqref{eq:b_t bound 3}, take the expectation with respect to all the sources of randomness and use the notation $ b_t = \E [ \Phi(\wbar_t) - f(\wbar_t, \Psi_t)]$ to conclude
\begin{align} \label{eq: b_t bound 8}
    b_{t+1}
    &\leq
    (1 - \mu_2 \eta_2 n) \E \left[ \Phi(\wbar_{t+1}) - f(\wbar_{t+1}, \Psi_{t}) \right] \\
    &\quad +
    \eta_2 L_{21}^2 e_t 
    +
    \eta_1^2 \eta_2 L_{21}^2 g_t 
     +
    \eta_1^2 \eta_2 L_{21}^2 \frac{\sigmaw}{n}
    +
    \eta_2^2 \frac{L_2}{2} \sigmapsi.
\end{align}
To bound the term $\E \left[ \Phi(\wbar_{t+1}) - f(\wbar_{t+1}, \Psi_{t}) \right]$, we can decompose it to the following three terms:
\begin{align} \label{eq:decompose1}
    \Phi(\wbar_{t+1}) - f(\wbar_{t+1}, \Psi_{t})
    &=
    \Phi(\wbar_{t}) - f(\wbar_{t}, \Psi_{t}) + f(\wbar_{t}, \Psi_{t}) - f(\wbar_{t+1}, \Psi_{t}) + \Phi(\wbar_{t+1}) - \Phi(\wbar_{t}).
\end{align}
Given the Lipschitz gradient assumption for the local functions in Assumption \ref{assumption:smooth} and using Lemma \ref{Lemma: f Lipschitz} on Lipschitz gradient for the global function, we can write
\begin{align} \label{eq: b_t bound 7}
    f(\wbar_{t}, \Psi_{t}) - f(\wbar_{t+1}, \Psi_{t})
    \leq
    - \langle \grw f(\wbar_t, \Psi_t), \wbar_{t+1} - \wbar_{t} \rangle 
    +
    \frac{L_1}{2} \norm{\wbar_{t+1} - \wbar_{t}}^2,
\end{align}
where $\wbar_{t+1} - \wbar_{t} = - \eta_1 \frac{1}{n} \sum_{i \in [n]} \sgrw f^i(\w^i_t , \bpsi^i_t)$. Taking expectation from both sides of \eqref{eq: b_t bound 7} implies that
\begin{align} \label{eq:decompose2}
    \E \left[ f(\wbar_{t}, \Psi_{t}) - f(\wbar_{t+1}, \Psi_{t}) \right]
    &\stackrel{(a)}{\leq}
    \eta_1 \E \norm{\grw f(\wbar_t, \Psi_t) - \gr \Phi(\wbar_t)}^2
    +
    \eta_1 \E \norm{\gr \Phi(\wbar_t)}^2 \\
    &\quad +
    \left( \frac{\eta_1}{2} + \eta_1^2 \frac{L_1}{2} \right) g_t
    +
    \eta_1^2 \frac{L_1}{2} \frac{\sigmaw}{n} \\
    &\stackrel{(b)}{\leq}
    \eta_1 \frac{2 L_{12}^2}{\mu_2 n} b_t
    +
    \eta_1 \E \norm{\gr \Phi(\wbar_t)}^2
    +
    \left( \frac{\eta_1}{2} + \eta_1^2 \frac{L_1}{2} \right) g_t
    +
    \eta_1^2 \frac{L_1}{2} \frac{\sigmaw}{n},
\end{align}
where in inequality $(a)$ we use the inequality $2 \langle \bba, \bbb \rangle \leq \norm{\bba}^2 + \norm{\bbb}^2$ and also the result in \eqref{eq:b_t bound 5}. To derive $(b)$, we use Assumptions \ref{assumption:smooth} and \ref{assumption:PL} (ii), result of Lemma \ref{Lemma: f Lipschitz} and the notation $ b_t = \E [ \Phi(\wbar_t) - f(\wbar_t, \Psi_t)]$  to write
\begin{align} 
    \E \norm{\gr \Phi(\wbar_t) - \grw f(\wbar_t, \Psi_t)}^2 
    &=
    \E \norm{\grw f(\wbar_t, \Psi^*(\wbar_t)) - \grw f(\wbar_t, \Psi_t)}^2 \\
    &\leq
    \frac{L_{12}^2}{n} \E \norm{\Psi^*(\wbar_t) - \Psi_t}^2_F \\
    &\leq
    \frac{2 L_{12}^2}{\mu_2 n} \E \left[ \Phi(\wbar_{t}) - f(\wbar_{t}, \Psi_{t}) \right] \\
    &=
    \frac{2 L_{12}^2}{\mu_2 n} b_t.
\end{align}
We now have all the ingredients to conclude the claim of Lemma \ref{Lemma: b_t contraction}. To do so, we combine the result of Lemma \ref{Lemma: Phi contraction} which bounds the term $\E [\Phi(\wbar_{t+1})] - \E [\Phi(\wbar_{t})]$, Lemma \ref{Lemma: h_t bound} that shows $h_t \leq 4 L_{12}^2 b_t / (\mu_2 n) + 2 L_1^2 e_t$, and the bound \eqref{eq:decompose2}; plug back in \eqref{eq:decompose1} and then in \eqref{eq: b_t bound 8} and conclude the claim of the lemma, that is
\begin{align}  
    b_{t+1} 
    &\leq
    (1 - \mu_2 \eta_2 n) \left( 1 + \eta_1 \frac{4 L_{12}^2}{\mu_2 n} \right) b_t
    +
    \frac{\eta_1}{2} \E \norm{\gr \Phi(\wbar_{t})}^2 
    +
    \frac{\eta_1^2}{2} \left( L_1 + L_{\Phi} + 2 \eta_2 L_{21}^2 \right) g_t \\
    &\quad +
    \left( \eta_1 L_1^2 + \eta_2 L_{21}^2 \right) e_t 
    +
    \frac{\eta_1^2}{2} \left( L_1 + L_{\Phi} + 2 \eta_2 L_{21}^2 \right) \frac{\sigmaw}{n} 
    +
    \frac{\eta_2^2}{2} L_2 \sigmapsi,
\end{align}

\subsection{Proof of Lemma \ref{lemma:e_t g_l bound}} \label{proof lemma:e_t g_l bound}
To prove this lemma, we first need to establish an intermediate step, which is stated in the following.
\begin{prop} \label{prop:e-t}
If Assumptions \ref{assumption:bounded-degree}, \ref{assumption:stoch-gradients} and \ref{assumption:smooth} hold, then
\begin{align}
    e_{t}
    &\leq
    16 \eta_1^2 (\tau - 1) L_1^2 \sum_{l=t_c+1}^{t-1} e_l 
    +
    10 \eta_1^2 (\tau - 1) \sum_{l=t_c+1}^{t-1} g_l
    +
    8 \eta_1^2 (\tau - 1)^2 \rho^2 
    +
    4 \eta_1^2 (\tau - 1) (n+1) \frac{\sigmaw}{n}.
\end{align}
\end{prop}
\begin{proof}[Proof of Proposition \ref{prop:e-t}]
Consider an iteration $t \geq 1$ and let $t_c$ denote the index of the most recent communication between the workers and the server, i.e. $t_c = \floor*{\frac{t}{\tau}} \tau$. Therefore, all the workers share the same local minimization model at iteration $t_c + 1$, i.e. $\w^1_{t_c+1}=\cdots=\w^n_{t_c+1}=\wbar_{t_c+1}$. According to the update rule of \texttt{FedRobust}, we can write for each node $i$ that
\begin{align} \label{eq: prop e_t 1}
    \w^i_{t_c+2} &= \w^i_{t_c+1} - \eta_1 \sgrw f^i(\w^i_{t_c+1}, \bpsi^i_{t_c+1}), \\
    \vdots \\
    \w^i_{t} &= \w^i_{t-1} - \eta_1 \sgrw f^i(\w^i_{t-1}, \bpsi^i_{t-1}).
\end{align}
Summing up all the equalities in \eqref{eq: prop e_t 1} yields that
\begin{align}
    \w^i_{t} &= \w^i_{t_c+1} - \eta_1 \sum_{l=t_c+1}^{t-1} \sgrw f^i(\w^i_l, \bpsi^i_l).
\end{align}
Therefore, the difference of the local models $\w^i_{t}$ and their average $\wbar_{t}$ can be written as
\begin{align}
    \w^i_{t} - \wbar_{t} 
    &= 
    \w^i_{t_c+1} - \eta_1 \sum_{l=t_c+1}^{t-1} \sgrw f^i(\w^i_l, \bpsi^i_l)
    -
    \left(\wbar_{t_c+1} - \eta_1 \frac{1}{n} \sum_{j \in [n]} \sum_{l=t_c+1}^{t-1} \sgrw f^j(\w^j_l, \bpsi^j_l) \right) \\
    &=
    -\eta_1 \left( \sum_{l=t_c+1}^{t-1} \sgrw f^i(\w^i_l, \bpsi^i_l) - \frac{1}{n} \sum_{j \in [n]} \sum_{l=t_c+1}^{t-1} \sgrw f^j(\w^j_l, \bpsi^j_l) \right).
\end{align}
This yields the following bound on each local deviation from the average $\E \Vert \w^i_{t} - \wbar_{t} \Vert^2$:
\begin{align} \label{eq: prop e_t 2}
    \E \norm{\w^i_{t} - \wbar_{t}}^2 
    &=
    \eta_1^2 \E \norm{\sum_{l=t_c+1}^{t-1} \sgrw f^i(\w^i_l, \bpsi^i_l) - \frac{1}{n} \sum_{j \in [n]} \sum_{l=t_c+1}^{t-1} \sgrw f^j(\w^j_l, \bpsi^j_l)}^2 \\
    &\leq
    2 \eta_1^2 
    \E \norm{\sum_{l=t_c+1}^{t-1} \sgrw f^i(\w^i_l, \bpsi^i_l) }^2
    +
    2 \eta_1^2 
    \E \norm{\frac{1}{n} \sum_{j \in [n]} \sum_{l=t_c+1}^{t-1} \sgrw f^j(\w^j_l, \bpsi^j_l)}^2\\
    &\stackrel{(a)}{\leq}
    2 \eta_1^2 
    \underbrace{\E \norm{\sum_{l=t_c+1}^{t-1} \grw f^i(\w^i_l, \bpsi^i_l) }^2 }_{T_3}
    +
    2 \eta_1^2 
    \underbrace{\E \norm{\frac{1}{n} \sum_{j \in [n]} \sum_{l=t_c+1}^{t-1} \grw f^j(\w^j_l, \bpsi^j_l)}^2 }_{T_4} \\
    &\quad +
    2 \eta_1^2 (t - t_c - 1) (n+1) \frac{\sigmaw}{n},
\end{align}
where we used Assumption \ref{assumption:stoch-gradients} to bound the variance of the stochastic gradients and derive $(a)$. The term $T_4$ in \eqref{eq: prop e_t 2} can simply be bounded as
\begin{align}
    T_4
    \leq
    \E \norm{\frac{1}{n} \sum_{j \in [n]} \sum_{l=t_c+1}^{t-1} \grw f^j(\w^j_l, \bpsi^j_l)}^2
    \leq
    (t - t_c - 1) \sum_{l=t_c+1}^{t-1} \E \norm{\frac{1}{n} \sum_{j \in [n]} \grw f^j(\w^j_l, \bpsi^j_l)}^2
\end{align}
Note that $t_c$ denotes the latest server-worker communication before iteration $t$, hence $t - t_c \leq \tau$ where $\tau$ is the duration of local updates in each round. Therefore, we have
\begin{align} \label{eq: prop e_t 4}
    T_4
    \leq
    (\tau - 1) \sum_{l=t_c+1}^{t-1} \E \norm{\frac{1}{n} \sum_{j \in [n]} \grw f^j(\w^j_l, \bpsi^j_l)}^2
    \leq
    (\tau - 1) \sum_{l=t_c+1}^{t-1} g_l
\end{align}
Now we proceed to bound the term $T_3$ in \eqref{eq: prop e_t 2} as follows:
\begin{align}
    T_3
    &=
    \E \norm{\sum_{l=t_c+1}^{t-1} \grw f^i(\w^i_l, \bpsi^i_l) }^2 \\
    &\leq
    (\tau - 1) \sum_{l=t_c+1}^{t-1} \E \norm{\grw f^i(\w^i_l, \bpsi^i_l) }^2 \\
    &\leq
    4 (\tau - 1) \sum_{l=t_c+1}^{t-1} \E \norm{\grw f^i(\w^i_l, \bpsi^i_l) - \grw f^i(\wbar_l, \bpsi^i_l) }^2 \\
    &\quad +
    4 (\tau - 1) \sum_{l=t_c+1}^{t-1} \E \norm{\grw f^i(\wbar_l, \bpsi^i_l) - \frac{1}{n} \sum_{j \in [n]} \grw f^j(\wbar_l, \bpsi^j_l) }^2 \\
    &\quad +
    4 (\tau - 1) \sum_{l=t_c+1}^{t-1} \E \norm{\frac{1}{n} \sum_{j \in [n]} \grw f^j(\wbar_l, \bpsi^j_l) - \frac{1}{n} \sum_{j \in [n]} \grw f^j(\w^j_l, \bpsi^j_l) }^2 \\
    &\quad +
    4 (\tau - 1) \sum_{l=t_c+1}^{t-1} \E \norm{\frac{1}{n} \sum_{j \in [n]} \grw f^j(\w^j_l, \bpsi^j_l) }^2
\end{align}
We can simply this bound by using Assumption \ref{assumption:smooth} on Lipschitz gradients for the local objectives $f^i$s and applying the notations for $e_l$ and $g_l$ to derive
\begin{align} \label{eq: prop e_t 3}
    T_3
    &\leq
    4 (\tau - 1) L_1^2 \sum_{l=t_c+1}^{t-1} \E \norm{\w^i_{l} - \wbar_{l}}^2  
    +
    4 (\tau - 1) \sum_{l=t_c+1}^{t-1} \E \norm{\grw f^i(\wbar_l, \bpsi^i_l) - \grw f(\wbar_l, \Psi_l) }^2 \\
    &\quad +
    4 (\tau - 1) L_1^2 \sum_{l=t_c+1}^{t-1} e_l
    +
    4 (\tau - 1) \sum_{l=t_c+1}^{t-1} g_l
\end{align}
We can plug \eqref{eq: prop e_t 4} and \eqref{eq: prop e_t 3} into \eqref{eq: prop e_t 2} and take the average of the both sides over $i = 1,\cdots,n$. This implies that
\begin{align}
    e_{t}
    &\leq
    16 \eta_1^2 (\tau - 1) L_1^2 \sum_{l=t_c+1}^{t-1} e_l 
    +
    10 \eta_1^2 (\tau - 1) \sum_{l=t_c+1}^{t-1} g_l
    +
    8 \eta_1^2 (\tau - 1)^2 \rho^2 
    +
    4 \eta_1^2 (\tau - 1) (n+1) \frac{\sigmaw}{n}.
\end{align}
In above, we used the result of Proposition \ref{prop: non-iid degree} that given Assumption \ref{assumption:bounded-degree}, bounds the gradient diversity $\frac{1}{n} \sum_{i \in [n]} \Vert{\grw f^i(\w, \bpsi^i) - \grw f(\w, \Psi)} \Vert^2 \leq \rho^2$, where $\rho^2 = 3\rho_f^2 + 6 L_{12}^2 (\epsilon_1^2 + \epsilon_2^2)$. We defer the proof this proposition to the end of this section. This concludes the proof of Proposition \ref{prop:e-t}.
\end{proof}
Having set the required intermediate steps, we resume the proof of Lemma \ref{lemma:e_t g_l bound}. According to Proposition \ref{prop:e-t}, we can write the term $e_t$ as follows
\begin{align} \label{eq:e_t}
    e_{t}
    &\leq
    C_1 \sum_{l=t_c+1}^{t-1} e_l 
    +
    C_2 \sum_{l=t_c+1}^{t-1} g_l
    +
    C_3
\end{align}
where we use the following short-hand coefficients
\begin{align} \label{eq: coefficients}
    C_1
    &\coloneqq
    16 \eta_1^2 (\tau - 1) L_1^2 \\
    C_2
    &\coloneqq
    10 \eta_1^2 (\tau - 1) \\
    C_3
    &\coloneqq
    8 \eta_1^2 (\tau - 1)^2 \rho^2 + 4 \eta_1^2 (\tau - 1) (n+1) \frac{\sigmaw}{n}.
\end{align}
We can then write this bound for every iteration in $[t_c+1 : t]$, that is
\begin{align}
    e_{t_c + 1}  
    &=
    0 \\
	e_{t_c + 2} 
	&\leq
	C_1 e_{t_c + 1} + C_2 g_{t_c + 1} + C_3 \\
	\vdots \\
	e_{t} 
	&\leq
	C_1 \left( e_{t_c + 1} + \cdots + e_{t-1} \right) + C_2 \left( g_{t_c + 1} + \cdots + g_{t-1} \right) + C_3 .
\end{align}
Summing all of the inequalities results in the following
\begin{align}
    \sum_{l=t_c+1}^{t-1} e_l
    \leq
    C_1 (\tau - 1) \sum_{l=t_c+1}^{t-1} e_l
    +
    C_2 (\tau - 1) \sum_{l=t_c+1}^{t-1} g_l
    +
    C_3 (\tau - 1).
\end{align}
We can further rearrange the terms above and write
\begin{align}
    \sum_{l=t_c+1}^{t-1} e_l
    \leq
    \frac{C_2 (\tau - 1)}{1 - C_1 (\tau - 1)} \sum_{l=t_c+1}^{t-1} g_l
    +
    \frac{C_3 (\tau - 1)}{1 - C_1 (\tau - 1)}.
\end{align}
Now, if we assume that $C_1 (\tau - 1) \leq 1/2$, then we get the following bound on $\sum_{l=t_c+1}^{t-1} e_l$
\begin{align}
    \sum_{l=t_c+1}^{t-1} e_l
    &\leq
    2 C_2 (\tau - 1) \sum_{l=t_c+1}^{t-1} g_l
    +
    2 C_3 (\tau - 1)
\end{align}
Plugging back in \eqref{eq:e_t} and using the assumption $C_1 (\tau - 1) \leq 1/2$ yields that
\begin{align}
    e_{t}
    &\leq
    C_1 \left( 2 C_2 (\tau - 1) \sum_{l=t_c+1}^{t-1} g_l
    +
    2 C_3 (\tau - 1) \right)
    +
    C_2 \sum_{l=t_c+1}^{t-1} g_l
    +
    C_3\\
    &\leq
    2 C_2 \sum_{l=t_c+1}^{t-1} g_l
    +
    2 C_3,
\end{align}
which concludes the proof of Lemma \ref{lemma:e_t g_l bound}. Lastly, we present the following proposition along with its proof which we used this result to prove Proposition \ref{prop:e-t}.
\begin{prop} \label{prop: non-iid degree}
An immediate implication of Assumptions \ref{assumption:bounded-degree} and \ref{assumption:smooth} is that for any $\w, \Psi$, the diversity of the local gradients is bounded in the following sense
\begin{align} 
    \frac{1}{n} \sum_{i \in [n]} \norm{\grw f^i(\w, \bpsi^i) - \grw f(\w, \Psi)}^2 
    &\leq
    \rho^2,
\end{align}
where we denote $\rho^2 = 3\rho_f^2 + 6 L_{12}^2 (\epsilon_1^2 + \epsilon_2^2)$.
\end{prop}
\begin{proof} [Proof of Proposition \ref{prop: non-iid degree}]
The proof is simply implied from Assumptions \ref{assumption:bounded-degree} and \ref{assumption:smooth} by writing
\begin{align} 
    \frac{1}{n} \sum_{i \in [n]} \norm{\grw f^i(\w, \bpsi^i) - \grw f(\w, \Psi)}^2 
    &\leq
    3 \frac{1}{n} \sum_{i \in [n]} \norm{\grw f^i(\w, \Lambda^i, \delta^i) - \grw f^i(\w, I, 0)}^2 \\
    &\quad +
    3 \frac{1}{n} \sum_{i \in [n]} \norm{\grw f^i(\w) - \grw f(\w) }^2 \\
    &\quad + 
    3 \frac{1}{n} \sum_{i \in [n]} \norm{\grw f(\w, I, 0) - \grw f(\w, \Psi)}^2 \\
    &\leq
    3\rho_f^2 + 6 L_{12}^2 (\epsilon_1^2 + \epsilon_2^2).
\end{align}
\end{proof}

\subsection{Proof of Lemma \ref{lemma:elimination}} \label{proof lemma:elimination}
\cite{haddadpour2019convergence} proves a similar claim for $\Gamma = 0$. For completeness, we provide the proof for general case when $ \Gamma \neq 0$. Let $t_c$ denote the index of the most recent communication round, i.e. $t_c = \floor*{\frac{t}{\tau}} \tau$. We can write $t = t_c + r$ where $1 \leq r \leq \tau$. Starting from $r=1$, we can write
\begin{align} 
    P_{t_c+2} 
    &\leq
    \Upsilon P_{t_c+1}
    -
    \frac{\eta_1}{2} \left( 1 - \eta_1 L \right) g_{t_c+1}  
    +
    \Gamma \\
    &\leq
    \Upsilon P_{t_c+1}
    +
    \Gamma,
\end{align}
where the last inequality holds if 
\begin{align} 
    \eta_1 L \leq 1.
\end{align}
We can continue for $r=2$ as follows
\begin{align} 
    P_{t_c+3} 
    &\leq
    \Upsilon P_{t_c+2}
    -
    \frac{\eta_1}{2} \left( 1 - \eta_1 L \right) g_{t_c+2}  
    +
    \eta_1^2 B g_{t_c+1} 
    +
    \Gamma \\
    &\stackrel{(a)}{\leq}
    \Upsilon^2 P_{t_c+1}
    -
    \frac{\eta_1}{2} \Upsilon \left( 1 - \eta_1 L - \eta_1 \frac{2B}{\Upsilon} \right) g_{t_c+1} 
    +
    \Gamma (1 + \Upsilon)\\
    &\stackrel{(b)}{\leq}
    \Upsilon^2 P_{t_c+1}
    +
    \Gamma (1 + \Upsilon) 
\end{align}
where $(a)$ is due to the inequality $P_{t_c+2} \leq \Upsilon P_{t_c+1} - \frac{\eta_1}{2} ( 1 - \eta_1 L ) g_{t_c+1} + \Gamma$ and $(b)$ holds if 
\begin{align} 
    1 - \eta_1 L - \eta_1 \frac{2B}{\Upsilon} \geq 0,
\end{align}
or equivalently
\begin{align} 
    \eta_1 \left( L + \frac{2B}{\Upsilon} \right) \leq 1.
\end{align}
We can continue the same argument up to $r+1$ and write
\begin{align} \label{eq: lemma P convergence 3}
    P_{t_c+r+1} 
    &\leq
    \Upsilon^{r} P_{t_c+1}
    +
    \Gamma (1 + \Upsilon + \cdots + \Upsilon^{r-1}),
\end{align}
if the step-size is as small as follows
\begin{align} 
    \eta_1 \left( L + \frac{2B}{\Upsilon^{r-1}} \left( 1 + \Upsilon + \cdots + \Upsilon^{r-2} \right) \right) \leq 1.
\end{align}
Since  $1 + \Upsilon + \cdots + \Upsilon^{r-2} \leq \frac{1}{1 - \Upsilon}$, then the following condition implies all the previous ones on $\eta$
\begin{align}  \label{eq: lemma P convergence 2}
    \eta_1 \left( L + \frac{2B}{\Upsilon^{r-1} (1 - \Upsilon)} \right).
\end{align}
Moreover, since $\Upsilon < 1$, then the strongest condition on $\eta$ is \eqref{eq: lemma P convergence 2} when we put the largest possible value for $r$ which is $\tau$, yielding 
\begin{align} 
    \eta_1 \left( L + \frac{2B}{\Upsilon^{\tau-1} (1 - \Upsilon)} \right).
\end{align}
Lastly, we note that $1 + \Upsilon + \cdots + \Upsilon^{r-1} \leq \frac{1}{1 - \Upsilon}$ in \eqref{eq: lemma P convergence 3}, and the claim is concluded.

\subsection{Proof of Lemma \ref{lemma: sum b_t}} \label{proof lemma: sum b_t}
Recall the result of Lemma \ref{Lemma: b_t contraction} in which we showed that if $\eta_2 \leq 1/L_2$, then the following contraction bound on the sequence $\{b_t\}_{t \geq 0}$ holds:
\begin{align} \label{eq: lemma sum b_t 1}
    b_{t+1} 
    &\leq
    (1 - \mu_2 \eta_2 n) \left( 1 + \eta_1 \frac{4 L_{12}^2}{\mu_2 n} \right) b_t
    +
    \frac{\eta_1}{2} \E \norm{\gr \Phi(\wbar_{t})}^2 
    +
    \frac{\eta_1^2}{2} \left( L_1 + L_{\Phi} + 2 \eta_2 L_{21}^2 \right) g_t \\
    &\quad +
    \left( \eta_1 L_1^2 + \eta_2 L_{21}^2 \right) e_t 
    +
    \frac{\eta_1^2}{2} \left( L_1 + L_{\Phi} + 2 \eta_2 L_{21}^2 \right) \frac{\sigmaw}{n} 
    +
    \frac{\eta_2^2}{2} L_2 \sigmapsi,
\end{align}
and consider the coefficient of $b_t$ in above. A simple calculation yields that if the step-sizes satisfy the condition $\frac{\eta_2}{\eta_1} \geq \frac{8 L_{12}^2}{\mu_2^2 n^2}$, then we have 
\begin{align} 
    (1 - \mu_2 \eta_2 n) \left( 1 + \eta_1 \frac{4 L_{12}^2}{\mu_2 n} \right)
    \leq
    1 - \frac{1}{2} \mu_2 \eta_2 n.
\end{align}
Now, we denote $\gamma = 1 - \frac{1}{2} \mu_2 \eta_2 n$ and apply \eqref{eq: lemma sum b_t 1} to all iterations $t = 0, \cdots, T-1$, which yields that
\begin{align}
    b_{0}
    & \leq
    \frac{2 L_{2}^2}{\mu_2 n} \left(\epsilon_1^2 + \epsilon_2^2 \right), \\
    b_{1} 
    &\leq
    \gamma b_0
    +
    \frac{\eta_1}{2} \E \norm{\gr \Phi(\wbar_{t})}^2 
    +
    \frac{\eta_1^2}{2} \left( L_1 + L_{\Phi} + 2 \eta_2 L_{21}^2 \right) g_0
    +
    \left( \eta_1 L_1^2 + \eta_2 L_{21}^2 \right) e_0 \\
    &\quad +
    \frac{\eta_1^2}{2} \left( L_1 + L_{\Phi} + 2 \eta_2 L_{21}^2 \right) \frac{\sigmaw}{n} 
    +
    \frac{\eta_2^2}{2} L_2 \sigmapsi, \\
    &\vdots\\
    b_{T-1} 
    &\leq
    \gamma b_{T-2}
    +
    \frac{\eta_1}{2} \E \norm{\gr \Phi(\wbar_{t})}^2 
    +
    \frac{\eta_1^2}{2} \left( L_1 + L_{\Phi} + 2 \eta_2 L_{21}^2 \right) g_{T-2}
    +
    \left( \eta_1 L_1^2 + \eta_2 L_{21}^2 \right) e_{T-2} \\
    &\quad +
    \frac{\eta_1^2}{2} \left( L_1 + L_{\Phi} + 2 \eta_2 L_{21}^2 \right) \frac{\sigmaw}{n} 
    +
    \frac{\eta_2^2}{2} L_2 \sigmapsi.
\end{align}
Taking the average of the $T$ inequalities above yields that
\begin{align} \label{eq: lemma sum b_t 2}
    (1 - \gamma) \frac{1}{T} \sum_{t=0}^{T-1} b_{t}
    &\leq
    \frac{2 L_{2}^2}{\mu_2 n} \frac{\epsilon_1^2 + \epsilon_2^2}{T}
    +
    \frac{\eta_1}{2} \frac{1}{T} \sum_{t=0}^{T-1} \E \norm{\gr \Phi(\wbar_{t})}^2 \\
    &\quad +
    \frac{\eta_1^2}{2} \left( L_1 + L_{\Phi} + 2 \eta_2 L_{21}^2 \right) \frac{1}{T} \sum_{t=0}^{T-1} g_{t}
    +
    \left( \eta_1 L_1^2 + \eta_2 L_{21}^2 \right) \frac{1}{T} \sum_{t=0}^{T-1} e_{t} \\
    &\quad +
    \frac{\eta_1^2}{2} \left( L_1 + L_{\Phi} + 2 \eta_2 L_{21}^2 \right) \frac{\sigmaw}{n} 
    +
    \frac{\eta_2^2}{2} L_2 \sigmapsi.
\end{align}
We can further divide both sides of \eqref{eq: lemma sum b_t 2} by $1 - \gamma$ and conclude
\begin{align} 
    \frac{1}{T} \sum_{t=0}^{T-1} b_{t}
    &\leq
    \frac{4 L_{2}^2}{\mu_2^2 n^2} \frac{\epsilon_1^2 + \epsilon_2^2}{\eta_2 T}
    +
    \frac{\eta_1}{\eta_2} \frac{1}{\mu_2 n} \frac{1}{T} \sum_{t=0}^{T-1} \E \norm{\gr \Phi(\wbar_{t})}^2 \\
    &\quad +
    \frac{\eta_1^2}{\eta_2} \frac{1}{\mu_2 n} \left( L_1 + L_{\Phi} + 2 \eta_2 L_{21}^2 \right) \frac{1}{T} \sum_{t=0}^{T-1} g_{t}
    +
    \frac{1}{\eta_2} \frac{2}{\mu_2 n} \left( \eta_1 L_1^2 + \eta_2 L_{21}^2 \right) \frac{1}{T} \sum_{t=0}^{T-1} e_{t} \\
    &\quad +
    \frac{\eta_1^2}{\eta_2} \frac{1}{\mu_2 n} \left( L_1 + L_{\Phi} + 2 \eta_2 L_{21}^2 \right) \frac{\sigmaw}{n} 
    +
    \eta_2 \frac{L_2}{\mu_2 n} \sigmapsi.
\end{align}

\subsection{Proof of Lemma \ref{lemma: sum e_t}} \label{proof lemma: sum e_t}

We begin by noting the result of Proposition \ref{prop:e-t} in which we showed the following bound on $e_t$
\begin{align} \label{eq:e_t}
    e_{t}
    &\leq
    C_1 \sum_{l=t_c+1}^{t-1} e_l 
    +
    C_2 \sum_{l=t_c+1}^{t-1} g_l
    +
    C_3,
\end{align}
where we defined the coefficients $C_1, C_2, C_3$ in \eqref{eq: coefficients} and recall here for more convenient:
\begin{align} 
    C_1
    &\coloneqq
    16 \eta_1^2 (\tau - 1) L_1^2 \\
    C_2
    &\coloneqq
    10 \eta_1^2 (\tau - 1) \\
    C_3
    &\coloneqq
    8 \eta_1^2 (\tau - 1)^2 \rho^2 + 4 \eta_1^2 (\tau - 1) (n+1) \frac{\sigmaw}{n}.
\end{align}
Next, we apply this bound to each iteration $t=0, \cdots, T-1$ as follows
\begin{align}
    &\quad e_0 = 0 \\
    &\left\{
	\begin{array}{ll}
		e_{1}  &= 0 \\
		e_{2} &\leq C_1 e_1 + C_2 g_1 + C_3 \\
		\vdots \\
		e_{\tau} &\leq C_1 \left( e_1 + \cdots + e_{\tau-1} \right) + C_2 \left( g_1 + \cdots + g_{\tau-1} \right) + C_3 
	\end{array}
    \right.\\
    &\left\{
	\begin{array}{ll}
		e_{\tau+1}  &= 0 \\
		e_{\tau+2} &\leq C_1 e_{\tau+1} + C_2 g_{\tau+1} + C_3 \\
		\vdots \\
		e_{2\tau} &\leq C_1 \left( e_{\tau+1} + \cdots + e_{2\tau-1} \right) + C_2 \left( g_{\tau+1} + \cdots + g_{2\tau-1} \right) + C_3  
	\end{array}
    \right.\\
    &\quad\quad \vdots\\
    &\left\{
	\begin{array}{ll}
		e_{T_c + 1}  &= 0 \\
		e_{T_c + 2} &\leq C_1 e_{T_c + 1} + C_2 g_{T_c + 1} + C_3 \\
		\vdots \\
		e_{T-1} &\leq C_1 \left( e_{T_c + 1} + \cdots + e_{T-2} \right) + C_2 \left( g_{T_c + 1} + \cdots + g_{T-2} \right) + C_3,
	\end{array}
    \right.
\end{align}
where $T_c = \floor*{\frac{T}{\tau}} \tau$ denote the index of the most recent communication between the workers and the server before iteration $T$. Summing the above inequalities yields that
\begin{align} \label{eq: lemma sum e_t 1}
    \sum_{t=0}^{T-1} e_{t}
    &\leq
    C_1 (\tau - 1) \sum_{t=0}^{T-1} e_{t} 
    +
    C_2 (\tau - 1) \sum_{t=0}^{T-1} g_t
    +
    C_3 T.
\end{align}
Now if we assume that $C_1 (\tau - 1) = 16 \eta_1^2 (\tau - 1)^2 L_1^2 \leq \frac{1}{2}$, the the claim is concluded by rearranging the terms in \eqref{eq: lemma sum e_t 1}: 
\begin{align}
    \frac{1}{T} \sum_{t=0}^{T-1} e_{t}
    &\leq
    2 C_2 (\tau - 1) \frac{1}{T} \sum_{t=0}^{T-1} g_t
    +
    2 C_3.
\end{align}

\section{Proof of Theorem \ref{Thm: generalization}} \label{appendix: proof of Thm generalization}

Fix a distribution $\tilde{P}$ and consider 
\begin{equation}
    \max_{\Lambda,\delta}\: \mathbb{E}_{\tilde{P}}[\ell(f_{\w}(\Lambda \mathbf{x}+\delta))] - \lambda \Vert\delta \Vert^2_2 - \lambda \Vert \Lambda - I\Vert^2_F   
\end{equation}
Assuming a $1$-Lipschitz loss $\ell$ with $1$-Lipschitz gradient, based on \cite{farnia2018generalizable}'s Lemma 7 the above function's gradient with respect to $\delta$ has a Lipschitz constant bounded by 
\begin{equation*}
    \operatorname{Lip}(\nabla f_{\w}):=\bigl(\prod_{i=1}^L \Vert\w_i\Vert_\sigma\bigr)\sum_{i=1}^l\prod_{j=1}^i\Vert \w_j\Vert_\sigma.
\end{equation*}
Similarly, the expected loss's derivative with respect to $\Lambda$ will also be Lipschitz in the spectral norm with a Lipschitz constant upper-bounded by  
\begin{equation*}
    B\operatorname{Lip}(\nabla f_{\w})=B\bigl(\prod_{i=1}^L \Vert\w_i\Vert_\sigma\bigr)\sum_{i=1}^l\prod_{j=1}^i\Vert \w_j\Vert_\sigma.
\end{equation*}
Given weights in $\w$, we denote the optimal solution for $\delta$ and $\Lambda$ by $\delta_{\w}$ and $\Lambda_{\w}$, respectively. To apply the Pac-Bayes generalization analysis, we need to bound the change in $\delta_{\w},\Lambda_{\w}$ caused by perturbing $\w$ to $\w + \bm{u}$. Note that since $\lambda>(1+B)\operatorname{Lip}(\nabla f_{\w})$, the maximization problem for optimizing $\Lambda_{\w}, \delta_{\w}$ is maximizing a strongly-concave objective whose solutions will satisfy:
\begin{align*}
    \delta_{\w} &=\frac{1}{\lambda} \mathbb{E}[\nabla\ell\circ f_{\w}(\Lambda_{\w}\mathbf{x}+ \delta_{\w})], \\
    {\Lambda}_{\w}-I &=\frac{1}{\lambda} \mathbb{E}[(\nabla\ell\circ f_{\w}(\Lambda_{\w}\mathbf{x}+ \delta_{\w}))\mathbf{X}^\top]
\end{align*}
which are norm-bounded by $\frac{\operatorname{Lip}(\ell \circ f_{\w})}{\lambda}\le \frac{\prod_{i=1}^d \Vert \w_i \Vert_\sigma }{\lambda}$ and $B\frac{\operatorname{Lip}(\ell \circ f_{\w})}{\lambda}\le B\frac{\prod_{i=1}^d \Vert \w_i \Vert_\sigma }{\lambda}$, respectively. Therefore, for a norm-bounded perturbation $\bm{u}$ where $\Vert\bm{u}_i\Vert_\sigma\le \frac{1}{L}\Vert\w_i\Vert_\sigma$ we can write
\begin{align*}
    &\big\Vert {\delta}_{\w + \bm{u}}- {\delta}_{\w} \big\Vert_2 + \big\Vert\Lambda_{\w + \bm{u}}- \Lambda_{\w} \big\Vert_\sigma  \\
    =\, & \big\Vert \frac{1}{\lambda}\mathbb{E}[\nabla \ell(f_{\w + \bm{u}}(\Lambda_{\w + \bm{u}}\mathbf{X} +\delta_{\w + \bm{u}} ))] - \frac{1}{\lambda}\mathbb{E}[\nabla \ell(f_{\w}(\Lambda_{\w}\mathbf{X} +\delta_{\w} ))] \big\Vert_2 \\
    &\quad + \big\Vert \frac{1}{\lambda}\mathbb{E}[\nabla \ell(f_{\w + \bm{u}}(\Lambda_{\w + \bm{u}}\mathbf{X} +\delta_{\w + \bm{u}} ))\mathbf{X}^\top] - \frac{1}{\lambda}\mathbb{E}[\nabla \ell(f_{\w}(\Lambda_{\w}\mathbf{X} +\delta_{\w} ))\mathbf{X}^\top] \big\Vert_\sigma \\
    =\, & \big\Vert \frac{1}{\lambda}\mathbb{E}[\nabla \ell(f_{\w + \bm{u}}(\Lambda_{\w + \bm{u}}\mathbf{X} +\delta_{\w + \bm{u}} )) - \nabla \ell(f_{\w}(\Lambda_{\w}\mathbf{X} +\delta_{\w} ))] \big\Vert_2 \\
    &\quad + \big\Vert \frac{1}{\lambda}\mathbb{E}[(\nabla \ell(f_{\w + \bm{u}}(\Lambda_{\w + \bm{u}}\mathbf{X} +\delta_{\w + \bm{u}} )) - \nabla \ell(f_{\w}(\Lambda_{\w}\mathbf{X} +\delta_{\w} )))\mathbf{X}^\top] \big\Vert_\sigma  \\
    \le \, & \big\Vert \frac{1}{\lambda}\mathbb{E}[\nabla \ell(f_{\w + \bm{u}}(\Lambda_{\w + \bm{u}}\mathbf{X} +\delta_{\w + \bm{u}} )) - \nabla \ell(f_{\w}(\Lambda_{\w + \bm{u}}\mathbf{X} +\delta_{\w + \bm{u}} ))] \big\Vert_2 \\
    & \quad + \big\Vert \frac{1}{\lambda}\mathbb{E}[\nabla \ell(f_{\w}(\Lambda_{\w + \bm{u}}\mathbf{X} +\delta_{\w + \bm{u}} )) - \nabla \ell(f_{\w}(\Lambda_{\w}\mathbf{X} +\delta_{\w + \bm{u}} ))] \big\Vert_2 \\
    & \quad + \big\Vert \frac{1}{\lambda}\mathbb{E}[\nabla \ell(f_{\w}(\Lambda_{\w}\mathbf{X} +\delta_{\w + \bm{u}} )) - \nabla \ell(f_{\w}(\Lambda_{\w}\mathbf{X} +\delta_{\w} ))] \big\Vert_2 \\
    & \quad + \big\Vert \frac{1}{\lambda}\mathbb{E}[(\nabla \ell(f_{\w + \bm{u}}(\Lambda_{\w + \bm{u}}\mathbf{X} +\delta_{\w + \bm{u}} )) - \nabla \ell(f_{\w}(\Lambda_{\w + \bm{u}}\mathbf{X} +\delta_{\w + \bm{u}} )))\mathbf{X}^\top] \big\Vert_\sigma \\
    & \quad + \big\Vert \frac{1}{\lambda}\mathbb{E}[(\nabla \ell(f_{\w}(\Lambda_{\w + \bm{u}}\mathbf{X} +\delta_{\w + \bm{u}} )) - \nabla \ell(f_{\w}(\Lambda_{\w}\mathbf{X} +\delta_{\w + \bm{u}} )))\mathbf{X}^\top] \big\Vert_\sigma \\
    & \quad + \big\Vert \frac{1}{\lambda}\mathbb{E}[(\nabla \ell(f_{\w}(\Lambda_{\w}\mathbf{X} +\delta_{\w + \bm{u}} )) - \nabla \ell(f_{\w}(\Lambda_{\w}\mathbf{X} +\delta_{\w} )))\mathbf{X}^\top] \big\Vert_\sigma \\
    \le \, & \frac{(B+1)\operatorname{lip}(\ell\circ f_{\w})}{\lambda}\bigl( \Vert {\delta}_{\w + \bm{u}}- {\delta}_{\w}\Vert_2 + \Vert\Lambda_{\w + \bm{u}}- \Lambda_{\w} \Vert_\sigma \bigr) \\
    & \quad + (B+1)e^2(\prod_{i=1}^L\Vert\w_i \Vert_\sigma)\sum_{i=1}^d\biggl[ \frac{\Vert \bm{u}_i\Vert_\sigma}{\Vert \w_i\Vert_\sigma} + B(\prod_{j=1}^i\Vert\w_j \Vert_\sigma)\sum_{j=1}^i\frac{\Vert\bm{u}_j\Vert_\sigma}{\Vert\w_j\Vert_\sigma}\biggr],
\end{align*}
where the last inequality follows from Lemma 3 in \cite{farnia2018generalizable}. As a result, 
\begin{align*}
     &\big\Vert {\delta}_{\w + \bm{u}} - {\delta}_{\w} \big\Vert_2 + \big\Vert\Lambda_{\w + \bm{u}}- \Lambda_{\w} \big\Vert_\sigma  \\
     \le \, &\frac{\lambda}{\lambda-(B+1)\operatorname{lip}(\ell\circ f_{\w})} \biggl[ (B+1)e^2(\prod_{i=1}^L\Vert\w_i \Vert_\sigma)\sum_{i=1}^d\bigl[ \frac{\Vert \bm{u}_i\Vert_\sigma}{\Vert \w_i\Vert_\sigma} + B(\prod_{j=1}^i\Vert\w_j \Vert_\sigma)\sum_{j=1}^i\frac{\Vert\bm{u}_j\Vert_\sigma}{\Vert\w_j\Vert_\sigma}\bigr]\biggr].
\end{align*}

Then, we can bound the change in the loss function caused by perturbing $\w$ at any $\Vert\mathbf{x}\Vert_2\le B$ with any norm-bounded $\Vert\bm{u}_i\Vert_\sigma\le\frac{1}{L}\Vert\w_i\Vert_\sigma$:
\begin{align*}
    &\big\Vert f_{\w + \bm{u}}(\Lambda_{\w + \bm{u}}\mathbf{X} +\delta_{\w + \bm{u}} ) - f_{\w}(\Lambda_{\w}\mathbf{X} +\delta_{\w} )\big\Vert_2 \\
    \le \, & \big\Vert f_{\w + \bm{u}}(\Lambda_{\w + \bm{u}}\mathbf{X} +\delta_{\w + \bm{u}} ) - f_{\w}(\Lambda_{\w + \bm{u}}\mathbf{X} +\delta_{\w + \bm{u}} )\big\Vert_2 \\
    &\, + \big\Vert f_{\w}(\Lambda_{\w + \bm{u}}\mathbf{X} +\delta_{\w + \bm{u}} ) - f_{\w}(\Lambda_{\w}\mathbf{X} +\delta_{\w + \bm{u}} )\big\Vert_2 \\
    &\, + \big\Vert f_{\w}(\Lambda_{\w}\mathbf{X} +\delta_{\w + \bm{u}} ) - f_{\w}(\Lambda_{\w}\mathbf{X} +\delta_{\w} )\big\Vert_2 \\
    \le \, & eB\bigl(\prod_{i=1}^L \Vert \w_i \Vert_\sigma \bigr)\sum_{i=1}^L\frac{\Vert \bm{u}_i\Vert_2}{\Vert \w_i\Vert_2} + (1+B)\bigl(\prod_{i=1}^d \Vert \w_i \Vert_\sigma \bigr) \\
    &\, \frac{e^2}{\lambda - (B+1)\operatorname{Lip}(\nabla f_{\w})}\sum_{i=1}^L \bigl[\frac{\Vert \bm{u}_i\Vert_\sigma}{\Vert \w_i\Vert_\sigma} +B(\prod_{j=1}^i \Vert \w_j \Vert_\sigma)\sum_{j=1}^i\frac{\Vert \bm{u}_j\Vert_\sigma}{\Vert \w_j\Vert_\sigma} \bigr]. 
\end{align*}
Now, for a fixed weight vector $\tilde{\w}$ we consider a multivariate Gaussian distribution $Q$ with zero-mean and diagonal covaraince matrix for perturbation $\bm{u}$ where each entry $\bm{u}_i$ has standard deviation $\kappa_i=\frac{\Vert \tilde{\w}_i\Vert_\sigma}{\sqrt[L]{\prod_{i=1}^L\Vert \tilde{\w}_i\Vert_\sigma}}\kappa$ with $\kappa$ chosen as
\begin{equation}
    \kappa = \frac{\gamma}{8e^5L\sqrt{2d\log(4dL)}B\bigl(\prod_{i=1}^L\Vert \tilde{\w}_i\Vert_\sigma\bigr)\bigl(1+\frac{\lambda}{\lambda-(1+B)\overline{\operatorname{Lip}}(\nabla  f_{\w})}\sum_{i=1}^L\prod_{j=1}^i\Vert\tilde{\w}_j\Vert_\sigma\bigr)}.
\end{equation}
Also, for any $\w$ which satisfies $|\Vert\w_i\Vert_\sigma - \Vert\tilde{\w}_i\Vert_\sigma |\le \frac{\eta}{4L}\Vert\tilde{\w}_i\Vert_\sigma$, we have $\overline{\operatorname{Lip}}(\ell\circ f_{\w})\le e^{\eta/2}\lambda(1-\eta)\le (1-\eta/2)\lambda$. Therefore,
\begin{align*}
    &\operatorname{KL}(P_{\w + \bm{u}}\Vert Q)\\
    \le\, & \sum_{i=1}^d\frac{\Vert\w_i\Vert^2_F}{2\kappa^2_i} \\
    \le \, & O\biggl( L^2B^2 d\log(dL)\frac{(\prod_{i=1}^L\Vert \tilde{\w}_i\Vert^2_\sigma)\bigl(1+\frac{1}{\lambda-(1+B)\overline{\operatorname{Lip}}(\nabla  f_{\w})}\sum_{i=1}^L\prod_{j=1}^i \Vert \tilde{\w}_j\Vert_\sigma \bigr)^2}{\gamma^2} \sum_{i=1}^d\frac{\Vert\w_i\Vert^2_F}{\Vert\tilde{\w}_i\Vert^2_\sigma}\biggr) \\
    \le \, & O\biggl( L^2B^2 d\log(dL)\frac{(\prod_{i=1}^L\Vert \w_i\Vert^2_\sigma)\bigl(1+\frac{1}{\lambda-(1+B)\overline{\operatorname{Lip}}(\nabla  f_{\w})}\sum_{i=1}^L\prod_{j=1}^i \Vert {\w}_j\Vert_\sigma \bigr)^2}{\gamma^2} \sum_{i=1}^d\frac{\Vert\w_i\Vert^2_F}{\Vert{\w}_i\Vert^2_\sigma}\biggr)
\end{align*}
Now we plug the above result into \cite{farnia2018generalizable}'s Lemma 1, implying that given a fixed underlying distribution $P$ and any $\xi>0$ with probability at least $1-\xi$ for any $\w$ satisfying $|\Vert\w_i\Vert_\sigma - \Vert\tilde{\w}_i\Vert_\sigma |\le \frac{\eta}{4L}\Vert \tilde{\w}_i\Vert_\sigma$ we have
\begin{equation}
    \mathcal{L}^{\operatorname{adv}}_{0-1}(\w) - \hat{\mathcal{L}}^{\operatorname{adv}}_{{\gamma}}(\w)
    \le
    \mathcal{O} \left( \sqrt{ \frac{ B^2L^2d\log(Ld)\lambda^2\bigl(\prod_{i=1}^L\Vert \w_i\Vert_{\sigma} \sum_{i=1}^L\frac{\Vert \w_i\Vert^2_F}{\Vert \w_i\Vert^2_{\sigma}}\bigr)^2+\log\frac{m}{\xi}  }{m \gamma^2 (\lambda-(1+B)\operatorname{Lip}(\nabla f_{\w}))^2} }\right).   
\end{equation}
Now we use a cover of size $O(\frac{L}{\eta}\log M)$ points where for any feasible $\Vert \w_i\Vert_\sigma$ we can find a point $a_i$ in the cover such that $|\Vert \w_i\Vert_\sigma- a_i|\le\frac{\eta}{4L}a_i$. As a result, we can cover the space of feasible $\w_i$'s with $O\bigl((\frac{L}{\eta}\log M))^L L \bigr)$ number of points. This proves that for a fixed underlying distribution for every $\xi>0$, with probability at least $\xi>0$ for any feasible norm-bounded $\w$ we have
\begin{equation}
    \mathcal{L}^{\operatorname{adv}}_{0-1}(\w) - \hat{\mathcal{L}}^{\operatorname{adv}}_{{\gamma}}(\w)
    \le
    \mathcal{O} \left( \sqrt{ \frac{ B^2L^2d\log(Ld)\lambda^2\bigl(\prod_{i=1}^L\Vert \w_i\Vert_{\sigma} \sum_{i=1}^L\frac{\Vert \w_i\Vert^2_2}{\Vert \w_i\Vert^2_{\sigma}}\bigr)^2+L\log\frac{mL\log(M)}{\eta\xi}  }{m \gamma^2 (\lambda-(1+B)\operatorname{Lip}(\nabla \ell \circ f_{\w}))^2} }\right).   
\end{equation}
To apply the result to the network of $n$ nodes, we apply a union bound to have the bound hold simultaneously for the distribution of every node, which proves for every $\xi>0$ with probability at least $1-\xi$ the average worst-case loss of the nodes satisfies the following margin-based bound: 
\begin{equation}
    \mathcal{L}^{\operatorname{adv}}_{0-1}(\w) - \hat{\mathcal{L}}^{\operatorname{adv}}_{{\gamma}}(\w)
    \le
    \mathcal{O} \left( \sqrt{ \frac{ B^2L^2d\log(Ld)\lambda^2\bigl(\prod_{i=1}^L\Vert \w_i\Vert_{\sigma} \sum_{i=1}^L\frac{\Vert \w_i\Vert^2_F}{\Vert \w_i\Vert^2_{\sigma}}\bigr)^2+L\log\frac{nmL\log(M)}{\eta\xi}  }{m \gamma^2 (\lambda-(1+B)\operatorname{Lip}(\nabla f_{\w}))^2} }\right).   
\end{equation}
Therefore, the proof is complete.

\section{Proof of Theorem \ref{Thm: DRO}} \label{appendix: proof of Thm DRO}
Define random vector $\mathbf{U}=\Lambda\mathbf{X}+\delta$. According to the definition of optimal transport cost $W_c(P_\mathbf{X},P_\mathbf{U})$ for quadratic $c(\mathbf{x},\mathbf{u})=\frac{1}{2}\Vert\mathbf{x}-\mathbf{u} \Vert^2_2$, 
\begin{equation}
    W_c(P_\mathbf{X},P_\mathbf{U}) := \min_{P_{\mathbf{X,U}}\in \Pi(P_\mathbf{X},P_\mathbf{U})}\: \mathbb{E}\bigl[ \frac{1}{2}\Vert\mathbf{X}-\mathbf{U} \Vert^2_2 \bigr]
\end{equation}
where $\Pi(P_\mathbf{X},P_\mathbf{U})$ contains any joint distribution $P_{\mathbf{X,U}}$ with marginals $P_\mathbf{X},P_\mathbf{U}$. One distribution in $\Pi(P_\mathbf{X},P_\mathbf{U})$ is the joint distribution of $(\mathbf{X},\Lambda\mathbf{X}+\delta)$ implying that
\begin{align*}
    W_c(P_\mathbf{X},P_\mathbf{U}) &\le \frac{1}{2}\mathbb{E}\bigl[\Vert \mathbf{X}-\Lambda\mathbf{X}-\delta\Vert_2^2\bigr] \\
    &= \frac{1}{2}\mathbb{E}\bigl[\Vert (I-\Lambda)\mathbf{X}-\delta\Vert_2^2\bigr] \\
    &\stackrel{(a)}{\le} \mathbb{E}\bigl[\Vert(I-\Lambda)\mathbf{X}\Vert^2_2\bigr] + \Vert\delta\Vert_2^2 \\
    &\stackrel{(b)}{\le} \operatorname{Tr}\bigl((I-\Lambda)(I-\Lambda)^\top\mathbb{E}[\mathbf{X}\mathbf{X}^\top]\bigr) + \Vert\delta\Vert_2^2 \\
    &\stackrel{(c)}{\le} \lambda \operatorname{Tr}\bigl((I-\Lambda)(I-\Lambda)^\top\bigr) + \Vert\delta\Vert_2^2 \\
    &\stackrel{(d)}{\le} \lambda \Vert I-\Lambda\Vert^2_F+ \Vert\delta\Vert_2^2 \\
    &\le \max\{\lambda,1\}\bigr( \Vert I-\Lambda\Vert^2_F+ \Vert\delta\Vert_2^2\bigl).
\end{align*}
In the above, $(a)$ holds since for every two vectors $\mathbf{u}_1,\mathbf{u}_2$ we have $\Vert \mathbf{u}_1 + \mathbf{u}_2\Vert^2_2 = \Vert \mathbf{u}_1\Vert^2_2 + \Vert\mathbf{u}_2\Vert^2_2 + 2\mathbf{u}^\top_1 \mathbf{u}_2 \le 2(\Vert \mathbf{u}_1\Vert^2_2 + \Vert\mathbf{u}_2\Vert^2_2)$. $(b)$ follows from the fact that $\mathbb{E}[\Vert(I-\Lambda)\mathbf{X}\Vert^2_2] = \mathbb{E}[\operatorname{Tr}((I-\Lambda)\mathbf{X}\mathbf{X}^\top(I-\Lambda)^\top) = \operatorname{Tr}\bigl((I-\Lambda)(I-\Lambda)^\top\mathbb{E}[\mathbf{X}\mathbf{X}^\top]\bigr)$. $(c)$ holds because of the theorem's assumption implying that $\mathbb{E}[\mathbf{X}\mathbf{X}^\top]\le \lambda I$. Last, $(d)$ holds because we have $\operatorname{Tr} (AA^\top) = \Vert A\Vert^2_F$ for every $A$. Therefore, the proof is complete.